\documentclass[lettersize,journal]{IEEEtran}
\usepackage{amsmath,amsfonts}
\usepackage{algorithmic}
\usepackage{algorithm}
\usepackage{array}
\usepackage[caption=false,font=normalsize,labelfont=sf,textfont=sf]{subfig}
\usepackage{textcomp}
\usepackage{stfloats}
\usepackage{url}
\usepackage{verbatim}
\usepackage{graphicx}
\usepackage{cite}

\usepackage{amssymb}
\usepackage{float}
\usepackage{multirow}
\usepackage{epsfig}
\usepackage{epstopdf}
\usepackage{makecell}
\usepackage{bm}
\usepackage{times}
\usepackage{booktabs}
\usepackage{microtype}
\usepackage{hyperref}
\usepackage{color}
\usepackage{diagbox}
\usepackage[english]{babel}
\usepackage{amsthm}
\newtheorem{theorem}{\textbf{Theorem}}
\theoremstyle{plain}

\begin{document}

\title{Robust  Graph Fine-Tuning with Adversarial Graph Prompting}

\author{Ziyan~Zhang, Bo~Jiang*\thanks{* Corresponding author} and Jin~Tang
	% <-this % stops a space
	\thanks{Ziyan Zhang and Jin Tang are with the School of Computer Science and Technology,  Anhui University, Hefei 230009, China (e-mail: zhangziyanahu@163.com, ahu\_tj@163.com)
		
		Bo Jiang is with
		State Key Laboratory of Opto-Electronic Information Acquisition and Protection Technology, School of Computer Science and Technology, Anhui University
		(e-mail: jiangbo@ahu.edu.cn)
	}
}

\markboth{Journal of \LaTeX\ Class Files,~Vol.~14, No.~8, August~2021}%
{Shell \MakeLowercase{\textit{et al.}}: A Sample Article Using IEEEtran.cls for IEEE Journals}

\IEEEpubid{0000--0000/00\$00.00~\copyright~2021 IEEE}

\maketitle

\begin{abstract} 
Parameter-Efficient Fine-Tuning (PEFT) method has emerged as a dominant paradigm for adapting pre-trained GNN models to downstream tasks. However, existing PEFT methods usually exhibit significant vulnerability to various noise and attacks on graph topology and node attributes/features. 
To address this issue, for the first time, we 
propose integrating  adversarial
learning 
 into graph prompting and develop a novel Adversarial Graph Prompting (AGP) framework 
to achieve robust graph fine-tuning. 
% There are two main aspects of the proposed 
Our AGP has two key aspects. 
\emph{First}, we propose the general problem formulation of AGP  as a min-max optimization problem and develop an alternating optimization scheme to solve it. 
For inner maximization, we propose Joint Projected Gradient Descent (JointPGD) algorithm to generate strong adversarial noise. 
For outer minimization, we employ a simple yet effective
module to learn the optimal node prompts to counteract the adversarial noise. 
\emph{Second}, 
we demonstrate that 
the proposed AGP can theoretically address both graph topology and node noise. This confirms
the versatility and robustness of our AGP fine-tuning method across various graph noise. 
%To solve it, we employ an alternating optimization algorithm to address the inner maximization and outer minimization subproblems. 
%For the inner maximization, we propose the Joint Projected Gradient Descent (JointPGD) strategy to generate strong adversarial noise. 
%For the outer minimization, we employ a simple prompt module to learn optimal prompts that counteract these hybrid adversarial perturbations. 
%Theoretical analysis demonstrates that our method effectively enhances model robustness against both topology and node feature noise. 
Note that, the proposed AGP is a general method that can be integrated with various pre-trained GNN models to enhance their robustness on the downstream tasks. 
Extensive experiments on multiple benchmark tasks validate the robustness and effectiveness of AGP method compared to state-of-the-art methods.

\end{abstract}


\section{Introduction}
\IEEEPARstart{R}{ecently}, Graph Neural Networks (GNNs) have emerged as the dominant approaches for learning and representing graph-structured data. They have been applied in a wide range of real-world applications, including recommendation systems~\cite{10387583}, molecular classification~\cite{GIN}, and computer vision~\cite{10638815}. Despite their success, the performance of GNNs typically depends on the availability of large-scale labeled datasets, which are often difficult and costly to obtain in practice. 
To alleviate this reliance, the widely used `pre-training \& fine-tuning' paradigm has been successfully extended to the graph learning field~\cite{graphcl,infograph,cuco}.
This paradigm first pre-trains GNNs on abundant unlabeled graph data to learn general structural patterns, and then adapts the models to downstream tasks via task-specific fine-tuning. However, full-parameter fine-tuning is computationally expensive and susceptible to catastrophic forgetting of the knowledge acquired during pre-training. 
% To address these limitations, a variety of parameter-efficient fine-tuning (PEFT) methods have been studied in recent years.
To tackle these issues, a broad range of parameter-efficient fine-tuning (PEFT) algorithms has been explored in recent years.

PEFT provides an efficient and flexible alternative to full-parameter tuning.
They aim to adapt the frozen pre-trained GNN models to the fine-tuning downstream tasks by introducing a few learnable parameters or lightweight modules.
In the graph domain, PEFT techniques can be broadly grouped into three representative paradigms, i.e., adapter-based, LoRA-based and prompt-based methods. 
Adapter-based methods generally incorporate some compact trainable adapters into the GNN architecture to capture task-specific knowledge while preserving the representations learned from the pre-training stage~\cite{AdapterGNN,gadapter}. 
Low-rank adaptation (LoRA) methods update the model through low-rank decompositions of weight matrices, which enables parameter updates restricted to low-dimensional subspaces~\cite{GraphLoRA}.
Prompt-based approaches introduce learnable prompts that are injected into node~\cite{GPF,UGP,GGPT}, edge~\cite{EdgePrompt}, or sub-graph~\cite{Allinone} to narrow the task or structural gap between pre-training and downstream tasks. Due to their simplicity and scalability, prompt-based methods have emerged as the most widely adopted PEFT strategy in graph learning. 
\begin{figure}[!htpb]
\centering
\includegraphics[width=0.5\textwidth]{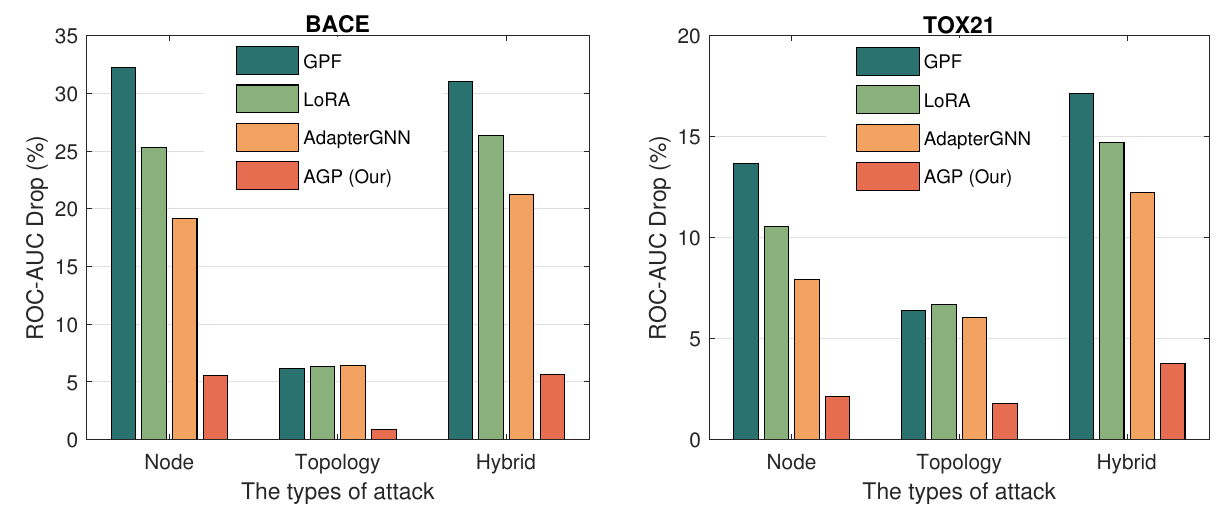}
 \caption{Comparison of ROC-AUC degradation under different adversarial attack targets (node, topology and hybrid) on BACE and TOX21 datasets. `ROC-AUC Drop' indicates the performance gap between clean and noisy baselines.
 Existing PEFT methods (GPF~\cite{GPF}, LoRA~\cite{hulora} and AdapterGNN~\cite{AdapterGNN}) show large drops across all attack types, while the proposed AGP exhibits consistently minimal degradation, demonstrating strong robustness against different types of adversarial noise. 
 }
\label{fig:acc_degradation}
\end{figure}
\IEEEpubidadjcol

%It is known that  
In many applications, graph data in downstream applications often suffer from
% in downstream tasks may contain 
various kinds of noise and attacks, such as spurious topologies or corrupted node features, stemming from environmental disturbances, human error, imperfect data processing, or malicious attacks, etc.~\cite{ASurvey}. 
%The core message-passing mechanism in pre-trained GNN models enables minor perturbations in graph topology and node features to propagate and amplify through layers, leading to severely degraded performance~\cite{mettack,nettack}. 
However, a critical challenge for pre-trained GNN models and existing PEFT methods is their vulnerability to adversarial perturbations and various noise in graph data. This usually leads to substantial performance degradation under attacks and noise on either graph topology or node features, as shown in Fig.~\ref{fig:acc_degradation}. 
% Existing studies on PEFT mainly focus on effective knowledge transfer from pre-trained GNN models to the downstream tasks. However, little attention has been devoted
% to {robust} learning on downstream tasks. 
Existing PEFT research primarily emphasizes the effectiveness of knowledge transfer from pre-trained GNNs. However, the aspect of robust learning within downstream fine-tuning scenarios remains under-explored largely.
%In real-world applications, various factors can introduce noise and degrade data quality, such as environment-induced natural disturbances, human recording errors, improper data processing, and malicious attacks~\cite{RGCN}.
%The core message-passing mechanism in GNNs enables minor perturbations in graph topology and node features to propagate and amplify through layers, leading to severely degraded performance~\cite{mettack,nettack}.
%This issue becomes even more pronounced in label-limited fine-tuning scenarios~\cite{song2025gpromptshield}, where the scarcity of supervised signals further amplifies the model’s vulnerability to such perturbations, as shown in Fig.~\ref{fig:acc_degradation}.
%It can be observed that all categories of PEFT methods exhibit substantial performance degradation under attacks targeting either the graph topology or node features. 
For robust PEFT, Song et al.~\cite{song2025gpromptshield} propose a robust graph prompt approach for fine-tuning against topology perturbation. 
%It constructs a shield defense system by deploying hybrid multi-defense prompts that adapt to specific filtering criteria, encompassing degree, node centrality similarity, and out-of-distribution characteristics.
However, %based on manual noise recognition rules, 
this approach detects the noise based on some manual noise recognition rules, which fails to generalize to some stealthy adversarial perturbations. 
Also, it is limited to address the topology noise, leaving node feature and hybrid noise unaddressed. 
\emph{Therefore, how to design a \textbf{robust} PEFT framework for graph fine-tuning to defend against various (topology and node) attacks and noise remains an open problem.} %this is still an unresolved issue: 

To bridge this gap,
we propose a new robust graph prompt learning scheme, termed  \textbf{A}dversarial \textbf{G}raph \textbf{P}rompting (AGP) by incorporating adversarial learning into the graph prompt tuning problem. 
Note that, adversarial learning strategies have been usually used in the traditional GNN learning field to enhance the robustness of GNN models~\cite{GraphPGD,IG}. 
However, to our best knowledge, 
it has not been employed for robust graph prompt learning problem. 
Specifically, 
\textbf{first}, we explicitly introduce the \emph{general problem formulation} of AGP 
as a min-max
optimization problem. 
%adversarial graph prompting}, termed  Adversarial Graph Prompting (AGP). 
%It consists of an inner maximization and an outer minimization subproblem.  
The inner maximization problem seeks to find strong adversarial noise that reduces the prompted model's performance, while the outer minimization problem aims to learn the robust graph prompts that counteract the effect of adversarial (topology and node) attacks/noise. % on the pre-trained GNN model. 
\textbf{Second}, 
%despite only learning node prompts, 
we demonstrate that
the proposed AGP can theoretically address both
graph topology and node noise, which confirms the versatility
and robustness of our AGP fine-tuning method across different types of
graph noise. 
\textbf{Third}, we propose an effective algorithm for our AGP problem, employing alternating optimization between inner and outer subproblems.   
For inner maximization, we derive an efficient Joint Projected Gradient Descent (JointPGD) algorithm to generate adversarial node and topology noise. 
For outer minimization, we design a simple yet effective module to learn the optimal node prompts. 

%Note that, the proposed AGP provides a general scheme which can be integrated with various pre-trained GNN models to enhance their robustness on the downstream tasks. 
%For outer optimization, we adopt the simple node feature prompting~\cite{GPF} and derive an efficient algorithm to learn robust prompts for graph nodes. 
%Interestingly, 
%despite only learning node prompts, we theoretically demonstrate that it can enhance the robustness of the fine-tuned GNN model against both topology and node feature noise. 
% Specifically, AGP 
% An effective algorithm has been derived to optimize the proposed adversarial prompt optimization problem. 
% To achieve this, we design a learnable graph prompt module and construct a robust graph prompt tuning framework. 
%More importantly, we provide theoretical analysis and intuitive illustrat  ions to justify the robustness of the learned prompts. 
%By alternating between the inner and outer optimization steps, AGP progressively enhances its robustness and effectively handles downstream graph data corrupted by adversarial attack noises.
% Through alternating iterations between this two problems, AGP becomes increasingly robust in handling downstream noisy data. 
% Experiments demonstrate that the proposed method significantly improves robustness against node feature, graph topology, and hybrid adversarial attack noises. 
% Overall, the main contributions of this paper are summarized as follows:
The primary contributions of this work are summarized in the following aspects:
\begin{itemize}
\item 
We integrate an adversarial learning strategy into the graph prompting problem and develop a novel Adversarial Graph Prompting (AGP) approach. % to achieve robust graph fine-tuning. 
AGP is a general scheme that can be integrated with various pre-trained GNN models to achieve robust fine-tuning on the downstream tasks. 

\item We demonstrate that the proposed AGP can theoretically address both graph topology and node noise, confirming the versatility and robustness of our AGP fine-tuning method across different types of graph noise. 

\item We derive an effective algorithm to solve the proposed AGP problem by employing
alternating optimization between inner and outer subproblems. 
For inner maximization, we derive a new Joint Projected Gradient Descent (JointPGD) algorithm to generate adversarial both topology and node noise efficiently.
% We propopse the general formaultion introduce Adversarial Graph Prompting (AGP), the first adversarial training optimization model designed for a graph prompt framework, to improve the robustness of GNN models against hybrid adversarial noises in the fine-tuning process.

%\item We optimize the proposed AGP model via the proposed hybrid attack strategy JointPGD and robust graph prompt module. 
%A theoretical analysis is provided to show that the generated prompt matrix can enhance the robustness of GNNs w.r.t. adversarial attack noises.

\end{itemize}

To evaluate the effectiveness of the proposed AGP method, we integrate it into several pre-trained GNN models on various downstream tasks.   
Extensive experiments on multiple benchmark tasks validate the effectiveness and robustness of AGP method compared to state-of-the-art 
graph fine-tuning methods.

%under various attack settings, demonstrating the effectiveness and superior robustness of the proposed AGP model.  

\section{Related works}

\subsection{Graph Prompt Tuning}
Prompt learning aims to adapt large pre-trained models to downstream tasks without full parameter fine-tuning~\cite{GPF,GPPT}, which has recently garnered increasing attention in the graph domain. 
Existing graph prompt learning techniques 
% can be broadly classified into two categories: 
can be generally partitioned into two distinct categories:
input-level and output-level prompting.
Input-level prompting aims to bridge the gap between pre-training and downstream tasks by aligning their data distributions. This is primarily achieved by injecting learnable components into the input space or intermediate GNN layers.
For instance, Graph Prompt Feature (GPF)~\cite{GPF} introduces shared prompt vectors that modulate input node features via a learnable weighting mechanism.
GraphControl~\cite{GraphControl} and EdgePrompt~\cite{EdgePrompt} optimize the input graph topology by incorporating edge-level prompts. 
All-in-One~\cite{Allinone} leverages informative subgraphs as prompts to adapt the model to specific tasks.  
Methods such as Generalized Graph Prompt (GGPT)~\cite{GGPT}, MultiGPrompt~\cite{MultiGPrompt}, and ProNoG~\cite{ProNoG} design task-specific prompt vectors that are combined with hidden layer features via element-wise addition or multiplication to steer the model toward target objectives.
Distinct from these input-level prompting, output-level prompting aligns pretext and downstream tasks by learning prototypes that establish a shared semantic space. Representative methods like GPPT~\cite{GPPT} and HetGPT~\cite{HetGPT} leverage this paradigm to adapt frozen pre-trained models to downstream tasks.
% Classifier-level prompting, in contrast, aims to bridge the gap between pretext and downstream objectives by learning a set of universal prototypes. These prototypes serve as a shared semantic space, allowing the fixed pre-trained model to adapt its output for downstream tasks without modifying its internal representations.
% Representative approaches, such as GPPT~\cite{GPPT} and HetGPT~\cite{HetGPT}, which leverages prototypical patterns to generalize across diverse learning objectives.
However, existing graph prompting frameworks remain inherently vulnerable to graph attacks and noise on either graph topology or node features, resulting in significant performance degradation, as illustrated in Fig.~\ref{fig:acc_degradation}. 
To overcome this limitation, we propose a robust adversarial graph prompt tuning framework.
% Methods like GPPT~\cite{GPPT} and HetGPT~\cite{HetGPT} fall under this category, where prototypical patterns are leveraged to generalize across different learning objectives.
% Existing graph prompt methods remain susceptible to performance degradation from graph noises and adversarial perturbations, as shown in Fig.~\ref{fig:acc_degradation}. To address this gap, we introduce a robust adversarial prompt tuning framework.

\subsection{Adversarial Training on Graph}
Adversarial training constitutes a fundamental defense strategy for enhancing GNN robustness against malicious attacks. 
% Its central premise involves formulating a min-max optimization objective, wherein model parameters are optimized to resist worst-case adversarial noise.
The key to adversarial training is formulating a min-max optimization objective, wherein model parameters are optimized to resist worst-case adversarial noise.
% improve model robustness under adversarial conditions by incorporating learnable adversarial noise during the training process. 
% Graph structural data can be perturbed in multiple ways, such as by corrupting node features with learnable noise~\cite{nettack}, adding/removing edges~\cite{GraphPGD,mettack,AGC}, injecting adversarial nodes~\cite{injectattack1,injectattack2}, or altering node labels~\cite{labelattack1,labelattack2}. This survey specifically investigates adversarial robustness within the paradigm of input-space additive noise attacks that do not introduce new nodes. This category includes common attack strategies, namely, perturbing node features and adding or deleting edges. 
According to the different attack targets, we can divide existing adversarial GNNs into three categories: feature-based, topology-based, and hybrid adversarial robust GNNs.
Feature-based adversarial training focuses on learning a perturbation matrix applied to node features. By injecting noise into the node feature space, these methods train GNNs to maintain performance under feature-level attacks, such as Nettack~\cite{nettack}.
Topology-based adversarial training only perturbs the graph topology through operations such as edge addition, deletion or rewiring. The model is then trained on perturbed graph topology to reduce its sensitivity to structural noise. Notable examples include Metattack~\cite{mettack}, GraphPGD~\cite{GraphPGD} and Graph Structure Attacking~\cite{AGC}.
Hybrid adversarial training simultaneously attacks both graph topology and node features. This approach creates more comprehensive and challenging adversarial examples, leading to models that are robust to various types of attack noises, such as IG-FGSM~\cite{IG} and Dual-targeted adversarial robust GNN~\cite{kwon2025dual}.
Despite their successes, existing adversarial robust GNN methods are designed for specific GNN architectures trained from scratch. 
They can not be directly applicable to the prevalent `pre-training \& fine-tuning' paradigm. 

% Lin et al.~\cite{backdoorGPT} use adversarial training to propose a backdoor attack framework designed specifically for GPL.
% However, it did not solve the robustness problem under the graph prompt framework.
While recent work~\cite{backdoorGPT} has utilized adversarial training to propose a new backdoor attack method within graph prompt learning framework, 
the challenge of ensuring adversarial robustness under the graph prompt framework remains an open problem. 
To overcome this problem, we present a novel graph prompt-based adversarial training framework for robust GNNs' fine-tuning under multiple attack scenarios.

\section{Adversarial Graph Prompt Model}

\subsection{Preliminaries}
Let $G(\mathbf{X}, \mathbf{A})$ be a graph with node features $\mathbf{X} \in \mathbb{R}^{N \times D}$ and adjacency matrix $\mathbf{A} \in \{0, 1\}^{N \times N}$, where $N$ and $D$ are the number of nodes and feature dimension, respectively. 
%An edge is present between nodes $i$ and $j$ if $\mathbf{A}_{ij} = 1$, and absent otherwise. 
% We consider a graph neural network model with learnable parameters set $\Theta$.
In scenarios with scarce downstream labels, the `pre-training \& fine-tuning' paradigm is commonly employed to enhance the expressive capability of graph models. 
Under this paradigm, the GNN backbone is first pre-trained on a source dataset via unsupervised manner~\cite{graphcl,infograph}, and subsequently adapted to downstream tasks through fine-tuning.
% The `pre-training \& fine-tuning' paradigm in the graph field typically performs in two stages: the GNN model is first pre-trained on extensive unlabeled graph data to capture general structural representations, after which the GNN parameters are fine-tuned on downstream tasks with limited labeled graph data.
Graph prompt learning is one of the effective fine-tuning methods~\cite {GPF,GGPT}, which can retain the pre-trained knowledge by optimizing 
only a small set of additional, parameterized graph prompts.
This paper focuses on the widely used input-level graph prompt method.
% is to add a prompt matrix $\mathbf{P}\in\mathbb{R}^{N\times D}$ to the node features, which has been proven in previous work~\cite{GPF} to be equivalent to performing operations on both features and structures simultaneously.
Let $\mathcal{G}(\cdot; \Theta^*, \mathcal{P})$ denote the prompt-augmented graph learning model,
which integrates a pre-trained backbone with fixed parameters $\Theta^*=\{\Theta^{(l)}\}_{l=0}^{L-1}$ and a set of learnable prompts $\mathcal{P}=\{\mathbf{P}^{(l)}\}_{l=0}^{L-1}$ to adapt the model to specific downstream tasks. Here, $\mathbf{P}^{(l)}$ represents the prompt component at the $l$-th layer, which can be instantiated in various forms, such as feature vectors, topological structures, or subgraphs.
Based on the above notations, the problem of input-level graph prompt learning can be formulated as, % the following minimization problem:
\begin{align}\label{eq:GPT_objective}
\min_{\mathcal{P}} \mathcal{L}_{task}\big(\mathcal{G}(\mathbf{X},\mathbf{A};\Theta^*,\mathcal{P}),\mathbf{Y}\big)
\end{align}
where $\mathcal{L}_{task}$ denotes the downstream task loss function and $\mathbf{Y}$ represents the labels.

\subsection{AGP model} 
% In scenarios with scarce downstream labels, the "pre-training \& fine-tuning" paradigm is commonly employed to enhance the expressive capability of graph models. 
% Within this framework, a GNN model is first pre-trained in an unsupervised manner~\cite{graphcl,infograph} on a source dataset. 
% The model is then adapted to downstream tasks via fine-tuning. 

In real-world scenarios, graph data on downstream tasks inevitably suffers from noise and malicious attacks. These issues, which manifest as topological anomalies or feature corruptions, arise from diverse sources such as environmental disturbances, human annotation errors, imperfect data preprocessing, and adversarial attacks~\cite{ASurvey}.
Pre-trained GNN models are typically sensitive to such graph noises or attacks, which may result in significant performance reduction on downstream tasks~\cite{song2025gpromptshield}.
% However, GNN model are often susceptible to noise or adversarial attacks in the downstream graph datasets.
% Moreover, under this training scheme, conducting full-parameter fine-tuning often leads to catastrophic forgetting of the knowledge acquired during pre-training~\cite{AdapterGNN}.
To overcome this issue, we develop a simple yet effective robust graph prompt finetuning method, named \textbf{A}dversarial \textbf{G}raph \textbf{P}rompt (AGP),
which introduces adversarial training into graph prompt for robust fine-tuning. 
Specifically, AGP formulates the graph prompt learning process as a min-max optimization problem. In this framework, the adversarial noise and graph prompts are updated iteratively in a mutually reinforcing manner: the perturbations are optimized to maximize the downstream task loss by simulating worst-case attacks, while the graph prompts are simultaneously trained to minimize this loss against such perturbations. This adversarial interplay forces the prompt model to be robust w.r.t. noise.
% Specifically, AGP first generates adversarial noises based on the learned graph prompt matrix and the pretrained GNN model.
% Subsequently, it fine-tunes only the parameters of graph prompts using the adversarial noises.

In this paper, we mainly focus on the formulation of adversarial attacks on node features and graph topology as additive noise components. Specifically. let $\mathbf{E}_{x}\in\mathbb{R}^{N\times D}$ and $\mathbf{E}_{a}\in\mathbb{R}^{N\times N}$ denote the adversarial perturbation matrices targeting node features and graph topology, respectively. Building upon the general formulation of graph prompt learning (Eq.(\ref{eq:GPT_objective})), our AGP is generally formulated as:
\begin{align}\label{eq:AGP_objective}
\begin{small}
\min_{\small{\mathcal{P}}}\max_{\small{\mathbf{E}_{x}\in\mathcal{C}_x,\mathbf{E}_{a}\in\mathcal{C}_a}} \mathcal{L}_{task}\big(\mathcal{G}(\mathbf{X}+\mathbf{E}_{x},\mathbf{A}+\mathbf{E}_{a};\Theta^*,\mathcal{P}),\mathbf{Y}\big)
\end{small}
\end{align}
where $\Theta^*$ represents the frozen parameters of the pre-trained GNN model.  $\mathcal{C}_x$ and $\mathcal{C}_a$ define the feasible perturbation constraints for node features and graph topology, respectively. 
Note that, the above AGP is a general scheme. It can be integrated with various pre-trained GNN models and specific prompting strategies to achieve robust fine-tuning on the downstream tasks. 

\section{Optimization of AGP model}
The optimization of the AGP model (Eq.(\ref{eq:AGP_objective})) constitutes a bi-level optimization problem consisting of: (i) {inner maximization}, which synthesizes adversarial noise subject to specific constraints to maximize the task loss and (ii) {outer minimization}, which optimizes the graph prompts to fortify the pre-trained GNN against such attacks. We derive an alternating  algorithm to optimize it. 
%Eq. (\ref{eq:AGP_objective}) by solving these two sub-problems in an alternating manner.
The detailed derivations for each step are presented in the subsequent subsections. 

\subsection{Inner maximization}
The inner maximization phase focuses on synthesizing worst-case perturbations. 
Its primary objective is to identify adversarial modifications to both graph topology and node features that maximize the downstream task loss. 
Fixing graph prompts $\mathcal{P}^*=\{\mathbf{P}^{(l)^*}\}^{L-1}_{l=0}$, this inner maximization sub-problem can be formally expressed as:
\begin{align}\label{eq:AGP_objective_max}
\max_{\mathbf{E}_{x}\in\mathcal{C}_x,\mathbf{E}_{a}\in\mathcal{C}_a} \mathcal{L}_{task}\big(\mathcal{G}(\mathbf{X}+\mathbf{E}_{x},\mathbf{A}+\mathbf{E}_{a};\Theta^*,\mathcal{P}^*),\mathbf{Y}\big)
\end{align}
Below, we first present the detail noise constraints and then derive an algorithm to optimize it. %, upon which the subsequent optimization method can be derived.

% where $\mathbf{X}^*_{adv}$ and $\mathbf{A}^*_{adv}$ represent the perturbed adversarial graph topology and node features, respectively.

\subsubsection{Constraints definition}

% To solve this problem, 
%We first unify the attack on graph topology and node features into the form of original data plus noise data.
%We first define the constraints $\mathcal{C}_x$ and $\mathcal{C}_a$ on noise matrices $\mathbf{E}_x$ and $\mathbf{E}_a$.
For feature-based attack, 
the noise matrix $\mathbf{E}_{x}\in\mathbb{R}^{N\times D}$ should be constrained to ensure its magnitude remains within a $\epsilon$-radius $q$-norm ball, thus guaranteeing the perturbation is imperceptible, as suggested in~\cite{PGD,PGD-trade}.
Thus, the constraints of node feature attack can be formulated as
\begin{align}\label{EQ:fpgd}
% \mathbf{X}_{adv} &= \mathbf{X} + \mathbf{E}_{x} \quad
% \text{s.t.}  \  
\mathbf{E}_{x}\in\mathcal{C}_{x}=\{\|\mathbf{E}_{x}\|_q \leq \epsilon\}
\end{align}
% where $\mathbf{X}_{adv}$ denotes the generated adversarial node features.
% In our experiments, the $\ell_\infty$ norm is adopted for $\|\cdot\|_q$.

For the discrete topology-based attack,
we introduce a binary indicator matrix $\mathbf{B} \in \{0, 1\}^{N \times N}$ to denote edge flips, where $\mathbf{B}_{ij} = \mathbf{B}_{ji} = 1$ denotes that the connection between nodes $i$ and $j$ is modified (edge removing or adding). This formulation allows us to model topology attacks as additive noise $\mathbf{E}_a$~\cite{GraphPGD}. %, aligning the optimization paradigm with continuous feature-based attack~\cite{GraphPGD}.
% In contrast to attacking the continuous node feature, attacks on discrete graph topology cannot be implemented via the straightforward addition of a continuous noise matrix.
% To solve this problem,
% we first introduce a discrete edge perturbation indicator matrix $\mathbf{B} \in \{0, 1\}^{N \times N}$, where $\mathbf{B}_{ij} = \mathbf{B}_{ji} = 1$ denotes that the connection between nodes $i$ and $j$ is modified (edge removing or adding).
% Then, the topology-based attack can be formulated under the same paradigm as the feature-based attack~\cite{GraphPGD}, i.e., adding a discrete noise matrix $\mathbf{E}_{a}$ to the original adjacency matrix $\mathbf{A}$.
To ensure the adversarial graph topology to be sparse, the edge perturbation indicator matrix $\mathbf{B}$ is assumed to  satisfy the sparse constraint. % that
% (Con1) the perturbation matrix $\mathbf{B}$ must remain discrete and
% the number of altered edges must not exceed $r\%$ of the total edges in the original graph. 
Consequently, the topology-based attack can be formulated as:
\begin{align}\label{EQ:spgd}
	% \mathbf{A}_{\mathrm{adv}} &= \mathbf{A} + \mathbf{E}_{a} = \mathbf{A} + \mathbf{B} \odot (\mathbf{1}\mathbf{1}^\top - \mathbf{I} - 2\mathbf{A}) \\
   \mathbf{E}_{a} &= \mathbf{B} \odot (\mathbf{1}\mathbf{1}^\top - \mathbf{I} - 2\mathbf{A}) \\
   {s.t.} \ \ \mathbf{B}\in\mathcal{C}_{a}&=\{\mathbf{B}_{ij} \in \{0, 1\},
	\|\mathbf{B}\|_0 \leq r \|\mathbf{A}\|_0\} \nonumber
\end{align}
where $\odot$ denotes the element-wise product operation. $\mathbf{1}$ denotes all-ones vector and $\mathbf{I}$ is the identity matrix. 
$\|\cdot\|_0$ denotes the $\ell_0$-norm, which counts the number of non-zero entries in the matrix.

\subsubsection{Objective optimization}
Based on the above constraint definition,   we can rewrite the above inner maximization sub-problem (Eq.(\ref{eq:AGP_objective_max})) as follows, 
\begin{align}\label{eq:AGP_objective_max_detail}
&\mathbf{E}^*_x,\mathbf{E}^*_a = \mathop{\arg\max}\limits_{\mathbf{E}_{x}\in\mathcal{C}_x,\mathbf{B}\in\mathcal{C}_a}
\mathcal{L}_{task}(\mathcal{G}(\mathbf{X}+\mathbf{E}_{x},\mathbf{A}+\mathbf{E}_{a};\Theta^*,\mathcal{P}^*),\mathbf{Y}) \nonumber\\
&{s.t.} \  \mathcal{C}_x = \{\|\mathbf{E}_{x}\|_q\leq\epsilon\},
\mathcal{C}_{a}=\{\mathbf{B}_{ij} \in \{0, 1\},
%\mathcal{C}_{a_2}: 
\|\mathbf{B}\|_0 \leq r \|\mathbf{A}\|_0\}
\end{align}
% where $\mathbf{E}^*_x,\mathbf{E}^*_a$ represent the final optimal
% adversarial noise matrices.
There are many attack methods that can be utilized to address this problem.
In this paper, we extend the standard Projected Gradient Descent (PGD)~\cite{PGD} to a new Joint Projected Gradient Descent (JointPGD), which is designed to concurrently attack both node features and graph topology.
% In this paper, we extend the single-target PGD method~\cite{PGD} to multi-targets , the problem (Eq.(\ref{eq:AGP_objective_max_detail})) requires attacking multiple objectives including graph topology and node features.
% A straightforward approach is to apply single target PGD method on graph topology and node features, respectively. 
% However, it requires double the computational cost and leads to sub-optimal solutions for the problem (Eq.(\ref{eq:AGP_objective_max_detail})). 
% To overcome these limitations, we propose a novel Joint Projected Gradient Descent (JointPGD) algorithm. 
Specifically, JointPGD simultaneously updates adversarial noise matrices of graph topology and node features by moving K steps in the direction of the joint gradient of the loss function. 
For simplicity, we denote $\mathbb{A} =\mathbf{1}\mathbf{1}^\top - \mathbf{I} - 2\mathbf{A}$ and let 
$\mathcal{L}^{(k)}$ denote 
\begin{equation}
    \mathcal{L}^{(k)}=\mathcal{L}_{task}(\mathcal{G}(\mathbf{X}+\mathbf{E}^{(k-1)}_{x},\mathbf{A}+\mathbf{E}^{(k-1)}_a);\Theta^*,\mathcal{P}^*),\mathbf{Y})
\end{equation}
where $k = 1,2,\dots, K$ denotes the perturbation steps.
Then, the $k$-th step of our JointPGD attack is formulated as:
\begin{align}
%	L^{(k)}&=\mathcal{L}_{task}(\mathcal{G}(\mathbf{X}+\mathbf{E}^{(k-1)}_{x},\mathbf{A}+\mathbf{E}^{(k-1)}_a);\Theta^*,\mathcal{P}^*),\mathbf{Y}) \\
	\mathbf{E}^{(k)}_{x} &= \mathbf{E}^{(k-1)}_{x}+\alpha \cdot {\rm sign}(\nabla_{\mathbf{X}^{(k)}_{adv}} \mathcal{L}^{(k)}) \label{eq:gradient1}\\
	\mathbf{E}^{(k)}_a&= (\mathbf{B}^{(k-1)}+\beta \cdot \nabla_{\mathbf{B}^{(k-1)}}\mathcal{L}^{(k)}) \odot \mathbb{A} \label{eq:gradient2}
\end{align}
where 
$\mathbf{X}^{(k)}_{adv} = \mathbf{X} + \mathbf{E}^{(k-1)}_x$ and 
$\alpha, \beta$ are learning rates. 
% $\nabla$ denotes the gradient operator and $\mathbf{X}^{(k)}_{adv} = \mathbf{X} + \mathbf{E}^{(k-1)}_x$.
% For node feature attacks, 
% we can apply the standard PGD algorithm~\cite{PGD} directly.
% While for graph topology attacks, we reference a discrete version of PGD~\cite{GraphPGD}.
% which constructs a continuous, learnable parameter matrix $\mathbf{B}$ to .
% Then, generating adversarial graph topology by iteratively maximizing the model loss as
%\begin{align}
%	\mathbf{X}^{(k-1)}_{adv} &= \mathbf{X} + \mathbf{E}^{(k-1)}_{x} \\
%	L^{(k)}&=\mathcal{L}_{task}(\mathcal{G}^{(l)}(\mathbf{X}^{(k-1)}_{adv}+\mathbf{P},\mathbf{A}),\mathbf{Y}) \\
%	\mathbf{X}^{(k)}_{adv} &= \mathbf{X}^{(k-1)}_{adv}+\alpha_x\cdot {\rm sign}(\nabla_{\mathbf{X}^{(k-1)}_{adv}} L^{(k)}) \\
%		\mathbf{E}^{(k)}_{x} &= min(-\epsilon, max(\mathbf{X}^{(k)}_{adv}-\mathbf{X}), \epsilon))
%\end{align}
% \begin{align}
% 	L^{(k)}&=\mathcal{L}_{task}(\mathcal{G}^{(l)}(\mathbf{X}+\mathbf{E}^{(k-1)}_{x}+\mathbf{P}^*,\mathbf{A}),\mathbf{Y}) \\
% 		\mathbf{E}^{(k)}_{x} &= \mathbf{E}^{(k-1)}_{x}+\alpha\cdot {\rm sign}(\nabla_{\mathbf{X}^{(k)}_{adv}} L^{(k)})
% 	%\mathbf{E}^{(k)}_{x} &= min(-\epsilon, max(\mathbf{E}^{(k-1)}_{x}+\alpha\cdot {\rm sign}(\nabla_{x} L^{(k)}), \epsilon))
% \end{align}
% where $\alpha$ denotes the learn rate and $k=\{1,2,\dots,K\}$ denotes the perturb step number. 
% The operator $\nabla_{x^{(k)}}$ denotes the gradient of adversarial node feature $\mathbf{X}^{(k)}_{adv}$, where $\mathbf{X}^{(k)}_{adv} = \mathbf{X}+\mathbf{E}^{(k-1)}_{x}+\mathbf{P}^*$.
%We initialize 
$\mathbf{E}^{(0)}_{x}$ is initialized as a random matrix following the uniform distribution $\mathcal{U}(-\epsilon, \epsilon)$ and $\mathbf{B}^{(0)}$ is initialized to a zero matrix. 
To ensure the constraint condition $\mathcal{C}_x$, we apply the following projection function after each iteration:
\begin{align}
	{\rm Proj}_x(\mathbf{E}_{x},\epsilon) = min(-\epsilon, max(\mathbf{E}_{x}, \epsilon))
\end{align}
% For graph topology attacks, we use a simplified discrete version of PGD~\cite{GraphPGD}.
% To be specific,
% this method first constructs a continuous, learnable parameter matrix $\mathbf{B}$.
% Then, generating adversarial graph topology by iteratively maximizing the model loss as:
% \begin{align}
% %\mathbf{E}^{(k)}_{a}&=\mathbf{B}^{(k-1)} \odot (\mathbf{1}\mathbf{1}^\top - \mathbf{I} - 2\mathbf{A}) \\
% L^{(k)}&=\mathcal{L}_{task}(\mathcal{G}^{(l)}(\mathbf{X}+\mathbf{P},\mathbf{A}+\mathbf{B}^{(k-1)} \odot \mathbb{A}),\mathbf{Y}) \\
% \mathbf{B}^{(k)}&= \mathbf{B}^{(k-1)}+\beta\cdot {\rm sign}(\nabla_{\mathbf{B}^{(k-1)}}L^{(k)})
% \end{align}
% where $\mathbb{A}=\mathbf{1}\mathbf{1}^\top - \mathbf{I} - 2\mathbf{A}$ and $\beta$ denotes the learn rate. 
The constraint $\mathcal{C}_a$ involves both $\ell_0$-norm constraint $\|\mathbf{B}\|_0 \leq r \|\mathbf{A}\|_0$ and discrete binary constraint $\mathbf{B}_{ij}\in\{0,1\}$. 
To ensure the $\ell_0$-norm constraint, 
%are enforced in two stages. 
%To satisfy constraint Con1,
we adopt a threshold function  after each iteration as
\begin{equation}
 {\rm Proj}_a(\mathbf{B}, r)_{ij}=\left\{
\begin{aligned}
	&0, & \ \  \mathbf{B}_{ij}< \eta \\
	&\mathbf{B}_{ij}, & \ \  \mathbf{B}_{ij} \geq \eta
\end{aligned}
\right.
\end{equation}
where $\eta$ denotes the $q$-th largest value in $\mathbf{B}$ and 
$q = \lfloor r\|\mathbf{A}\|_0\rfloor$. 
To ensure the binary constraint, 
we perform Bernoulli sampling~\cite{7446349} on $\mathbf{B}^{(K)}$ from the last iteration to obtain the final discrete perturbation indicator matrix $\mathbf{B}^*$. 
% For the more challenging hybrid attack, a straightforward approach is to apply the two aforementioned PGD methods independently to perturb graph topology and node features. However, this has a high computational cost and often leads to sub-optimal solutions for the problem (Eq.(\ref{eq:AGP_objective_max_detail})). 
% To overcome this limitation, we propose a novel Joint Projected Gradient Descent (JointPGD) algorithm. 
% It simultaneously updates adversarial perturbations for both graph topology and node features by moving in the direction of the combined gradient of the loss function in each iteration. 
The complete attack procedure JointPGD is summarized in Algorithm \ref{A:hyper attack}.

\begin{algorithm}[h]
\caption{JointPGD for hybrid attack }\label{A:fpdg}
\label{A:hyper attack}
\begin{algorithmic}[1]
\STATE \textbf{Input:} Graph data $G(\mathbf{X},\mathbf{A})$, graph labels $\mathbf{Y}$, learning rate $\alpha$ and $\beta$, perturbation radius $\epsilon$, perturbation ratio $r$, perturbation step $K$, prompt matrix $\mathcal{P}^*$, pre-trained GNN model with parameter $\Theta^*$\\
\STATE \textbf{Initialization:} Initialize noise and indicator matrix:\\
\quad $\mathbf{E}_{x}^{(0)} \sim \mathcal{U}(-\epsilon, \epsilon)$, 
$\mathbf{B}^{(0)} \gets \mathbf{0}$ \\
\FOR{$k = 1$ \TO $K$}
% \STATE Construct adversarial samples: \\
% \quad$\mathbf{X}_{{adv}}^{(k)} \gets \mathbf{X} + \mathbf{E}_{x}^{(k-1)}$,
% \quad$\mathbf{A}_{{adv}}^{(k)} \gets \mathbf{A} + \mathbf{B}^{(k-1)} \odot \mathbb{A}$
% \STATE Compute task loss on adversarial inputs:\\
% \quad$L^{(k)} \gets \mathcal{L}_{{task}}\left(\mathcal{G}(\mathbf{X}_{{adv}}^{(k)}, \mathbf{A}_{{adv}}^{(k)}; \Theta^*,\mathcal{P}^*), \mathbf{Y}\right)$
\STATE Perform gradient descent: \\
\quad $\mathbf{E}^{(k)}_{x} \gets\mathbf{E}^{(k-1)}_{x}+\alpha \cdot {\rm sign}(\nabla_{\mathbf{X}_{{adv}}^{(k)}} \mathcal{L}^{(k)})$\\
\quad  $\mathbf{B}^{(k)}\gets  \mathbf{B}^{(k-1)}+\beta \cdot\nabla_{\mathbf{B}^{(k-1)}}\mathcal{L}^{(k)}$ \\
\STATE Update noises via projection function: \\
\quad $\mathbf{E}^{(k)}_{x} \gets {\rm Proj}_x(\mathbf{E}^{(k)}_{x},\epsilon)$,
\  $\mathbf{B}^{(k)}\gets {\rm Proj}_a( \mathbf{B}^{(k)},r)$ \\
\ENDFOR
% \STATE Binarize topology structural perturbations: \\
% \quad$\mathbf{B}^* \gets \operatorname{Bernoulli}(\mathbf{B}^{(K)})$
\STATE Generate final adversarial graph noises: \\ %with Bernoulli sampling:
\quad $\mathbf{E}^*_{x}\gets\mathbf{E}_{x}^{(K)}$,
$\mathbf{E}^*_{a}\gets \operatorname{Bernoulli}(\mathbf{B}^{(K)}) \odot \mathbb{A}$
\STATE  \textbf{Return} Adversarial graph noises $\mathbf{E}^*_{x}$ and $\mathbf{E}^*_{a}$
\end{algorithmic}
\end{algorithm}

\begin{figure*}[!htpb]
\centering
\includegraphics[width=1.0\textwidth]{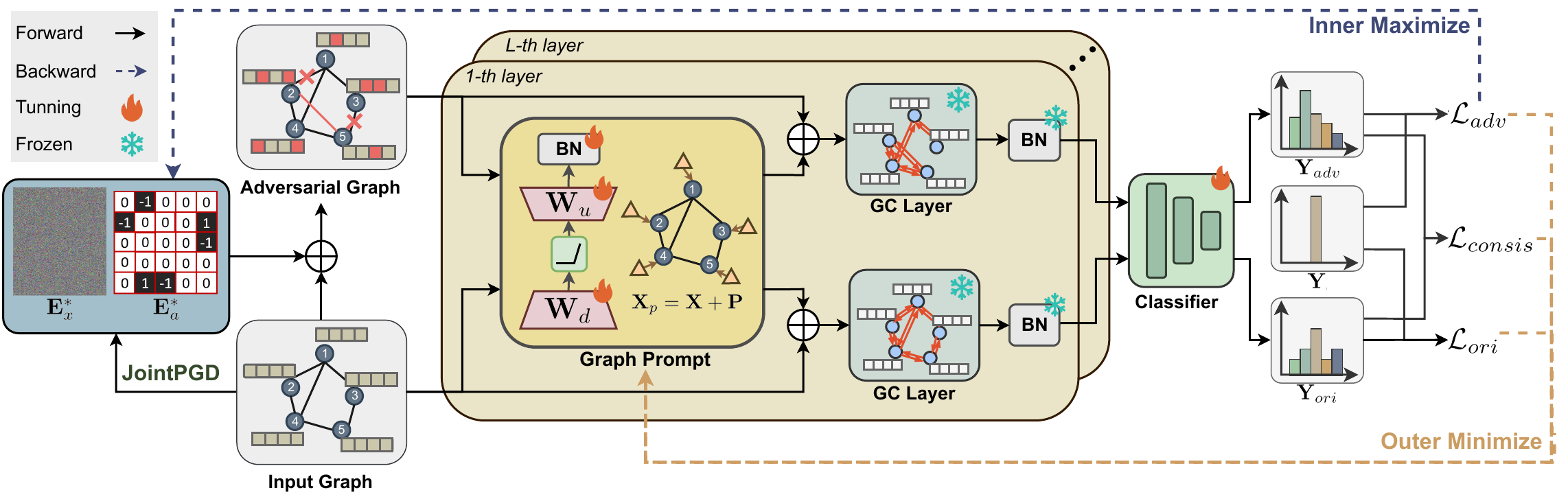}
 \caption{Overall architecture of the proposed Adversarial Graph Prompt (AGP) framework.
  The finetuning process involves two objectives: (1) Maximization (blue dashed line): maximizing the adversarial loss $\mathcal{L}_{adv}$ by generating more challenging adversarial noises $\mathbf{E}^*_x$ and $\mathbf{E}^*_a$ via JointPGD algorithm. (2) Minimization (red dashed line): minimizing the adversarial loss $\mathcal{L}_{adv}$, original loss $\mathcal{L}_{ori}$ and consistency loss $\mathcal{L}_{consis}$ to tune the prompt module and classifier while keeping the GNN backbone frozen.
 }
\label{fig:architecture}
\end{figure*}
\subsection{Outer minimization}
The outer minimization problem aims to learn robust graph prompts that enable the model to maintain performance under adversarial attacks, thereby minimizing the worst-case task loss.
Given the optimized adversarial graph noise $\mathbf{E}^*_{x}$ and $\mathbf{E}^*_{a}$, the outer minimization sub-problem can be formulated as:
\begin{align}\label{eq:AGP_objective_min}
\min_{\mathcal{P}} \mathcal{L}_{task}\big(\mathcal{G}(\mathbf{X}+\mathbf{E}^*_{x},\mathbf{A}+\mathbf{E}^*_{a};\Theta^*,\mathcal{P}),\mathbf{Y}\big)
\end{align}
% To optimize the outer minimization problem, we should begins with the design of a reliable graph prompt module, and then develop the solution based on it.
Note that, Eq.(\ref{eq:AGP_objective_min}) provides a general prompt learning formulation and various specific graph prompt strategies can be adopted here. 
%Our proposed framework (Eq.\ref{eq:AGP_objective}) is compatible with various graph prompt methods. 
%Motivated by insights from prior studies~\cite{GPF}, 
In this paper, we adopt the graph node prompt strategy in  which each node is augmented by adding a prompt vector to its feature in each pre-trained GNN layer, i.e.,  
% Following the finding in prior work~\cite{GPF} that integrating prompts into node features yields a more unified formulation and better generalizability, we adopt a simple additive strategy:
\begin{align}\label{eq:prompt_add}
\mathbf{X}^{(l)}_p = \mathbf{X}^{(l)} + \mathbf{P}^{(l)}
\end{align}
$\mathbf{P}^{(l)}$ denotes the learnable graph node prompts for the $l$-th layer where $l=0, 1,\cdots, L-1$ and $\mathbf{X}^{(0)}=\mathbf{X}$. 
Note that, this strategy can generally unify diverse graph prompting paradigms~\cite{GPF}. 
In particular, we will provide a theoretical analysis in $\S$ V to demonstrate its capability to mitigate both topological perturbation and node feature noise. 
Here, the prompt $\mathbf{P}^{(l)}$ in Eq.(\ref{eq:prompt_add}) can be learned either directly~\cite{GPF,GGPT} or using parameterized function~\cite{li2024instanceawaregraphpromptlearning}. %  through various mechanisms, such as static free parameters or dynamic generation via learnable modules. 
In this paper, we introduce a more compact and lightweight bottleneck function to compute it as
\begin{align}\label{eq:prompt_define}
\mathbf{P}^{(l)} = \text{BatchNorm}\left( \sigma\big( \mathbf{X}^{(l)} \mathbf{W}^{(l)}_{d} \big) \mathbf{W}^{(l)}_{u} \right)
\end{align}
where $\mathbf{W}^{(l)}_{d} \in \mathbb{R}^{D^{(l)} \times d}$, $\mathbf{W}^{(l)}_{u} \in \mathbb{R}^{d \times D^{(l)}}$ and $d \ll D^{(l)}$. % is the bottleneck dimension.
$\sigma(\cdot)$ is a nonlinear activation function and $\text{BatchNorm}(\cdot)$ denotes batch normalization operation. 
Fig.~\ref{fig:architecture} demonstrates the architecture of the proposed graph prompt learning module. 
%$\rm{MP}(\cdot)$ denotes a message passing function. In experiments, we use the message passing function of GIN~\cite{GIN}. 
% The proposed graph prompt module enhances the robustness of pre-trained Graph Neural Networks (GNNs) against diverse graph noises by generating instance-aware prompts that adapt to different noise perturbations on nodes and edges. 
% Through adversarial training, these tailored prompt matrices effectively mitigate the adverse effects of noise on model performance.
% This capability is theoretically analyzed and intuitively demonstrated in $\S$ III.C. 
% Furthermore, the differences between our method and other robust graph prompt methods and adversarial GNN are discussed in $\S$ III.D. 
%
Based on Eq.(\ref{eq:prompt_define}), the outer minimization in Eq.(\ref{eq:AGP_objective_min}) becomes to learn the optimal parameters $\mathcal{W}=\{\mathbf{W}^{(l)}_{d},\mathbf{W}^{(l)}_{u} \}^{L-1}_{l=0}$ that preserve model performance under adversarial noise $\mathbf{E}^*_{x}$ and $\mathbf{E}^*_{a}$, i.e., 
% Concretely, $\mathcal{W}$ is optimized via the following objective:
\begin{small}
\begin{align}
\label{eq:AGP_objective_min_detail}
\mathcal{W}^* = \arg\min_{\mathcal{W}} \mathcal{L}_{task}\bigl( \mathcal{G}\bigl( \mathbf{X}+\mathbf{E}^*_{x}, \mathbf{A}+\mathbf{E}^*_{a};\Theta^*, \mathcal{W}\bigr), \mathbf{Y} \bigr) 
\end{align}
\end{small}
This can be solved via gradient descent algorithm as, %tly.
%Specifically, the prompt parameters $\mathcal{W}$ are updated at each epoch as:
\begin{align}
\label{eq:AGP_objective_min_detail}
\mathcal{W}' = \mathcal{W} - \lambda\nabla_{\mathcal{W}} \mathcal{L}_{total}
\end{align}
where %$\mathcal{W}'$ denotes the updated prompt parameters. 
$\lambda$ denotes the learning rate.
The overall loss function $\mathcal{L}_{total}$ is formulated as
\begin{align}
\label{eq:loss_total}
\mathcal{L}_{total} =  \mathcal{L}_{adv} + \gamma \mathcal{L}_{ori} +\eta \mathcal{L}_{consis}
\end{align}
where $\gamma$ and $\eta$ are hyper-parameters.
$\mathcal{L}_{adv}$ denotes the adversarial loss defined as:
\begin{align}\label{eq:loss_adv}
\mathcal{L}_{adv} = \mathcal{L}_{task}(\mathbf{Y}_{adv}, \mathbf{Y})
\end{align}
where $\mathbf{Y}_{adv}$ denotes the model prediction based on adversarial inputs.
To preserve the model's standard inference capability during robust learning,
the model is jointly optimized with the task loss on the original clean data as
\begin{align}
\label{eq:loss_ori}
\mathcal{L}_{ori} = \mathcal{L}_{task}(\mathbf{Y}_{ori}, \mathbf{Y})
\end{align}
where $\mathbf{Y}_{ori}$ denotes the prediction from the original graph. 
In addition, as suggested in work~\cite{PGD-trade}, we also introduce a consistency loss $\mathcal{L}_{consis}$ to enforce consistency between the predictions of original and adversarial inputs, i.e., 
\begin{align}
\label{eq:loss_consis}
\mathcal{L}_{consis} = \mathcal{L}_{task}(\mathbf{Y}_{adv}, \mathbf{Y}_{ori}). 
\end{align}
% $\mathcal{L}{consis} = \mathcal{L}{task}(\mathbf{Y}{adv}, \mathbf{Y}{ori})$, 
The complete training procedure is shown in Algorithm~\ref{A:rgp} and
the overall framework of AGP is demonstrated in Fig.~\ref{fig:architecture}.
 
% $\mathcal{L}_{ori}$ represents the task loss on the original graph $\mathcal{G}(\mathbf{X}, \mathbf{A})$, which preserves the model's standard inference capability during robust learning. 

\begin{algorithm}[h]
\caption{Adversarial Graph Prompting}\label{A:rgp}
\begin{algorithmic}[1]
\STATE\textbf{Input:} Graph data $G(\mathbf{X}, \mathbf{A})$, graph label $\mathbf{Y}$, pre-trained GNN parameters $\Theta^*$, learning rate $\gamma$,
initial prompt parameters $\mathcal{W}^{(0)}$, the number of iterations $T$, hyperparameters $\gamma$, $\eta$
\FOR{$t = 1$ \TO $T$}
    \STATE Compute predictions on clean graph: \\
        \quad$\mathbf{Y}_{ori} \gets \mathcal{G}(\mathbf{X}, \mathbf{A};\Theta^*,\mathcal{W}^{(t-1)})$
    \STATE Generate adversarial noises via inner maximization: \\
        \quad$\mathbf{E}^*_{x},\mathbf{E}^*_{a} \gets \text{JointPGD}(G(\mathbf{X}, \mathbf{A}), \mathcal{G}(\cdot;\Theta^*,\mathcal{W}^{(t-1)}))$
    \STATE Construct adversarial samples:\\
   \quad $\mathbf{X}^*_{adv}\gets\mathbf{X}+\mathbf{E}^*_{x}$, $ \mathbf{A}^*_{adv}\gets\mathbf{A}+\mathbf{E}^*_{a}$
    \STATE Compute predictions on adversarial graph:\\
        \quad$\mathbf{Y}_{adv} \gets \mathcal{G}(\mathbf{X}^*_{adv}, \mathbf{A}^*_{adv};\Theta^*,\mathcal{W}^{(t-1)}))$
    \STATE Compute total loss: \\
        \quad$\mathcal{L}_{total} \gets \mathcal{L}_{adv} + \gamma\mathcal{L}_{consis} + \eta\mathcal{L}_{ori}$
    \STATE Update prompt parameters via gradient descent:\\
        \quad$\mathcal{W}^{(t)} \gets \mathcal{W}^{(t-1)} - \lambda\nabla_{\mathcal{W}^{(t-1)}} \mathcal{L}_{total}$
\ENDFOR
\STATE \textbf{Return:} Optimized prompt parameters $\mathcal{W}^{(T)}$
\end{algorithmic}
\end{algorithm}

\section{Theoretical analysis}

\begin{figure*}[!htpb]
\centering
\includegraphics[width=0.85\textwidth]{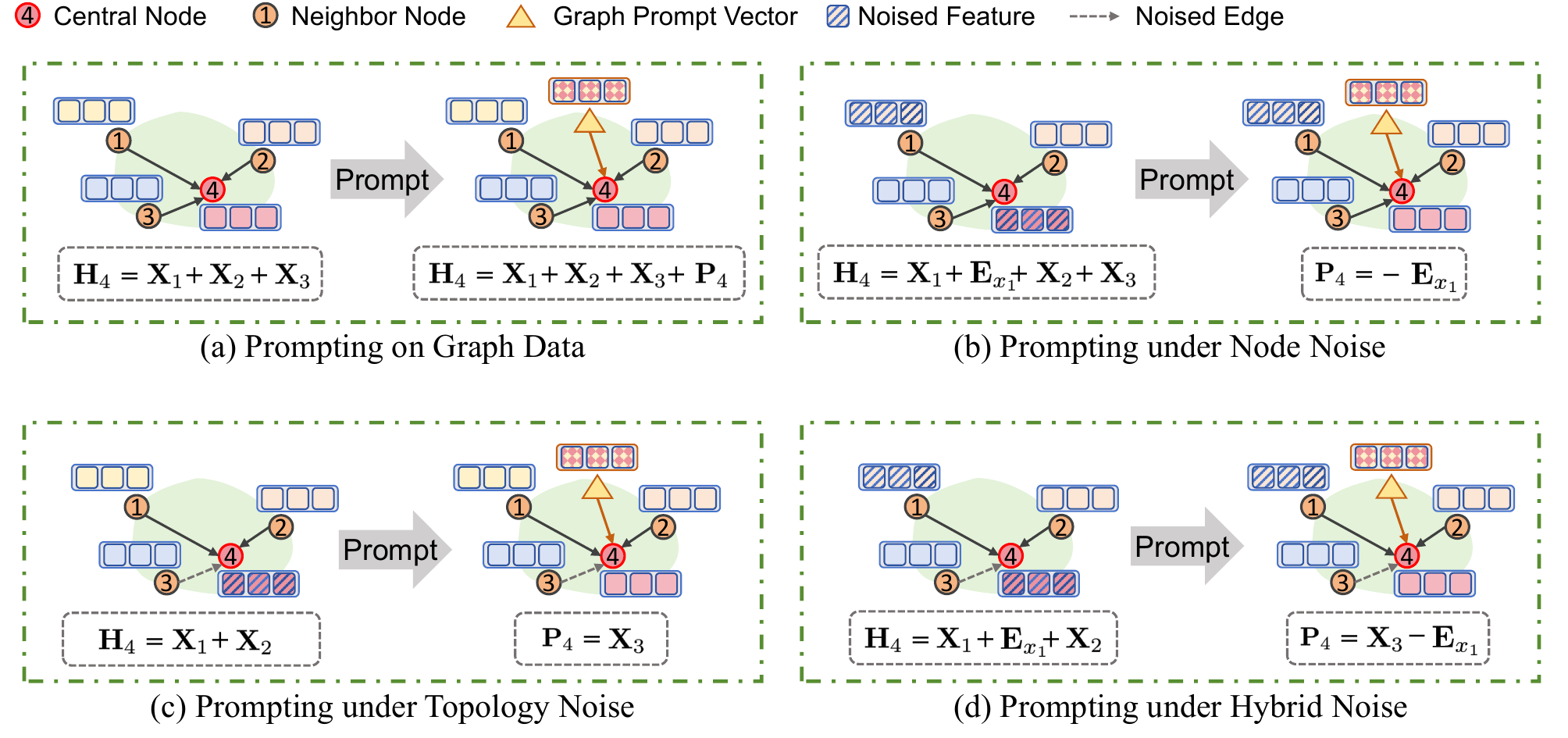}
  \caption{Illustrative examples demonstrating the capability of the proposed graph prompt in mitigating node and topology noise during neighborhood aggregation.}
\label{fig:demo}
\end{figure*}

% This section theoretically demonstrates that integrating an appropriate prompt matrix $\mathbf{P}$ with the node features can improve robustness to noises in $\mathbf{X}$ and $\mathbf{A}$.
In this section, we present a theoretical analysis and rigorously demonstrate how the proposed AGP can address both node and topology noise. 
Specifically, we have

\begin{theorem}\label{t3}
Suppose the input graph data $G(\hat{\mathbf{X}}, \hat{\mathbf{A}} )$ 
is corrupted by both node feature and topology noise $\{\mathbf{E}_x, \mathbf{E}_a\}$. For pre-training GNN model $\mathcal{G}(\cdot)$ with pre-trained parameters $\Theta^*=\{\Theta^{(l)}\}^{L-1}_{l=0}$, there exists optimal node prompts $\mathcal{P}=\{\mathbf{P}^{(l)}\}^{L-1}_{l=0}$ such that: 
\begin{equation}\label{EQ:EXSP}
\mathcal{G}(\hat{\mathbf{A}}, \hat{\mathbf{X}};\Theta^*,\mathcal{P}) = \mathcal{G}(\hat{\mathbf{A}} - \mathbf{E}_a, \hat{\mathbf{X}} - \mathbf{E}_x;\Theta^*)
\end{equation}
\end{theorem}

\begin{proof}
In the following proof, we take pre-trained GIN~\cite{GIN} as an example. Some other pre-trained model, such as GCN~\cite{GCN}, GraphSAGE~\cite{graphsage}, can be similarly obtained.  
For simplicity, let $\widetilde{\mathbf{A}} = \hat{\mathbf{A}} + (1+\varepsilon)\mathbf{I}$
where $\mathbf{I}$ denotes the identity matrix and $\varepsilon$ is a learnable scalar parameter.
Given the initial input 
$\hat{\mathbf{X}}^{(0)}=\hat{\mathbf{X}}$, the layer-wise propagation of GIN~\cite{GIN} with the node prompts $\mathbf{P}^{(l)}$
can be expressed as
\begin{align}\label{EQ:EXPL0}
&\mathcal{G}^{(l)}(\hat{\mathbf{A}},\hat{\mathbf{X}}^{(l)};\Theta^{(l)},\mathbf{P}^{(l)})
= \widetilde{\mathbf{A}}(\hat{\mathbf{X}}^{(l)} + \mathbf{P}^{(l)}) \Theta^{(l)}  \\
&= \widetilde{\mathbf{A}} \hat{\mathbf{X}}^{(l)} \Theta^{(l)}
   + \widetilde{\mathbf{A}} \mathbf{P}^{(l)} \Theta^{(l)} = \mathbf{X}^{(l+1)} + \widetilde{\mathbf{A}} \mathbf{P}^{(l)} \Theta^{(l)}\nonumber
\end{align}
%where % $\mathbf{X}^{(l+1)} = \widetilde{\mathbf{A}} \hat{\mathbf{X}}^{(l)} \mathbf{W}^{(l)}$
%denotes the original representation, and $\Delta \mathbf{X}^{(l+1)}_{p} = \widetilde{\mathbf{A}} \mathbf{P}^{(l)} \mathbf{W}^{(l)}$
%corresponds to the prompting term at the $l$-th layer.

% As shown in Eq.(\ref{EQ:EXPL0}), the GIN architecture adopts a node-wise update mechanism,
% where the corrupted node features $\hat{\mathbf{X}}$ are only involved at the input layer, while the corrupted topology $\hat{\mathbf{A}}$ participates in message propagation at every layer. Consequently, to ensure that Eq.~(\ref{EQ:EXSP}) holds, the input-layer prompt $\mathbf{P}^{(0)}$ and the intermediate-layer prompts $\{\mathbf{P}^{(l)}\}_{l=1}^{L-1}$ must compensate for different types of noise.

First, let's begin with the input layer ($l=0$). 
To ensure Eq.(\ref{EQ:EXSP}) holds, the optimal input-layer prompt $\mathbf{P}^{(0)}$ 
is required to satisfy the following equation:
% needs to make the following equation hold:
\begin{equation}\label{eq:gin_input_prompt}
\mathcal{G}^{(0)}(\hat{\mathbf{A}}, \hat{\mathbf{X}}^{(0)};\Theta^{(0)},\mathbf{P}^{(0)})
= \mathcal{G}^{(0)}(\hat{\mathbf{A}}-\mathbf{E}_a,\hat{\mathbf{X}}^{(0)}-\mathbf{E}_x)
\end{equation}
According to Eq.(\ref{EQ:EXPL0}), the left term of Eq.(\ref{eq:gin_input_prompt}) is
\begin{equation}\label{eq:gin_input}
\mathcal{G}^{(0)}(\hat{\mathbf{A}}, \hat{\mathbf{X}}^{(0)};\Theta^{(0)},\mathbf{P}^{(0)})
= \mathbf{X}^{(1)} +\widetilde{\mathbf{A}} \mathbf{P}^{(0)} \Theta^{(0)}
\end{equation}
The right term of  Eq.(\ref{eq:gin_input_prompt}) can be expanded as
\begin{align}\label{EQ:EXSPR}
	&\mathcal{G}^{(0)}(\hat{\mathbf{A}}-\mathbf{E}_{a},\hat{\mathbf{X}}^{(0)}-\mathbf{E}_{x};\Theta^{(0)})
	 =(\widetilde{\mathbf{A}}-\mathbf{E}_a) (\hat{\mathbf{X}}^{(0)}-\mathbf{E}_{x})  \Theta^{(0)} \nonumber\\
	& =\mathbf{X}^{(1)}-\mathbf{E}_a  \hat{\mathbf{X}}^{(0)}  \Theta^{(0)}-\widetilde{\mathbf{A}}  \mathbf{E}_{x}  \Theta^{(0)}+\mathbf{E}_a  \mathbf{E}_{x}  \Theta^{(0)}% \nonumber\\
%	& =\mathbf{X}^{(1)}+\Delta \mathbf{X}^{(1)}_e
\end{align}
By aligning Eq.(\ref{eq:gin_input}) and Eq.(\ref{EQ:EXSPR}), we can observe that Eq.(\ref{eq:gin_input_prompt}) holds when %$\Delta \mathbf{X}^{(1)}_{p} = \Delta \mathbf{X}^{(1)}_e$, i.e., 
\begin{equation}
    \widetilde{\mathbf{A}} \mathbf{P} \Theta^{(0)} = \mathbf{E}_a \mathbf{E}_{x} \Theta^{(0)} - \mathbf{E}_a \hat{\mathbf{X}}^{(0)} \Theta^{(0)} - \widetilde{\mathbf{A}} \mathbf{E}_{x} \Theta^{(0)}
\end{equation}
Thus, we can obtain the optimal  $\mathbf{P}^{(0)}$ as 
\begin{equation}\label{EQ:EXSPP0}
    \mathbf{P}^{(0)} = \widetilde{\mathbf{A}}^{-1} \mathbf{E}_a (\mathbf{E}_{x} - \hat{\mathbf{X}}^{(0)}) - \mathbf{E}_{x}
\end{equation}

Second, 
based on the above analysis, we can show that the node feature noise $\mathbf{E}_x$ is absorbed via the input prompt $\mathbf{P}^{(0)}$. Thus, for the following middle hidden layers, we only need to consider the topology noise $\mathbf{E}_a$ because the (corrupted) topology $\mathbf{\hat{A}}$ participates in all layers of GIN model, i.e., Eq.(\ref{EQ:EXSP}) becomes:
\begin{equation}\label{eq:EXSPLR}
\mathcal{G}^{(l)}(\hat{\mathbf{A}}, \hat{\mathbf{X}}^{(l)} ;\Theta^{(l)},\mathbf{P}^{(l)})
=
\mathcal{G}^{(l)}(\hat{\mathbf{A}} - \mathbf{E}_a, \hat{\mathbf{X}}^{(l)})
\end{equation}
where $l=1,\dots, L-1$. Following the similar derivation, Eq.(\ref{eq:EXSPLR}) holds when
\begin{equation}\label{EQ:EXSPPl}
    \mathbf{P}^{(l)} = -\widetilde{\mathbf{A}}^{-1} \mathbf{E}_a \hat{\mathbf{X}}^{(l)},\  l=1,\cdots, L-1
\end{equation}

Finally, based on 
the above Eq.(\ref{EQ:EXSPP0}) and Eq.(\ref{EQ:EXSPPl}), 
%The above derivations (Eqs.(\ref{EQ:EXSPP0},\ref{EQ:EXSPPl})) indicate that
we can find an optimal node prompt set $\mathcal{P}$ to compensate for both topology and node noise in pre-trained GIN model. 
\end{proof}

\textbf{Remark.} 
% We can extend Theorem~\ref{t3} to the case that the graph data is subjected to a single-modality noise attack.
% Specifically, when noise is restricted to the graph topology (i.e., $\mathbf{E}_x = \mathbf{0}$), 
% there exists a set of node prompts $\mathcal{P}=\{\mathbf{P}^{(l)}\}^{L-1}_{l=0}$ with each
% $\mathbf{P}^{(l)} = -\widetilde{\mathbf{A}}^{-1} \mathbf{E}_a \hat{\mathbf{X}}^{(l)}$ that compensates for the topological perturbation.
As analyzed  in the above proof, 
when both node noise $\mathbf{E}_x$ and topology noise  $\mathbf{E}_a$ exist, we need to conduct prompting on all pre-trained GNN layers, as shown in Eq.(\ref{eq:prompt_add}).  
However, when only node noise $\mathbf{E}_x$ 
 exists, 
% ($\mathbf{E}_a=\mathbf{0}$),
its effect can be entirely absorbed by setting $\mathbf{P}^{(0)} = -\mathbf{E}_x$ and $\mathbf{P}^{(l)} = \mathbf{0}$, $l=1,\dots, L-1$. 
That is, we can 
conduct a prompt only on the input layer. 
In the experiments, we denote this simple variant as AGP-S.

\textbf{Demonstration.} 
To provide an intuitive understanding of how the proposed prompting mechanism mitigates both topology and node noise, we illustrate the neighborhood aggregation process for a representative target node, denoted as Node 4 in Fig.~\ref{fig:demo}. 
Fig.~\ref{fig:demo} (a) shows the standard aggregation process enhanced by a prompt. 
Fig.~\ref{fig:demo} (b)-(d) demonstrate specific robustness scenarios:
(b) \emph{Node noise}: When the features of neighbor Node 1 are corrupted by a noise vector $\mathbf{E}_{x_1}$, the prompt $\mathbf{P}_4$ neutralizes this perturbation by setting $\mathbf{P}_4=-\mathbf{E}_{x_1}$; 
(c) \emph{Topology noise}: 
When the edge between Node 4 and Node 3 is cut,
the prompt $\mathbf{P}_4$ effectively restores the missing connection by setting $\mathbf{P}_4 = \mathbf{X}_3$; 
(d) \emph{Hybrid noise}: Under a hybrid attack involving both node and topology noise, the prompt $\mathbf{P}_4 = \mathbf{X}_3 -\mathbf{E}_{x_1}$ simultaneously counteracts both kinds of noise.
The above cases intuitively demonstrate the theoretical findings in Theorem~\ref{t3}, showing that an optimal prompt can effectively preserve representation integrity against diverse noise.

\section{Experiments}
To comprehensively evaluate the robustness and effectiveness of the proposed AGP, we benchmark it against five representative fine-tuning methods across seven molecular datasets under three distinct graph attack scenarios.
\subsection{Experimental Setup}
\subsubsection{Datasets}
The backbone model is pre-trained on 2 million unlabeled molecular graphs which are sampled from the ZINC15 database~\cite{ZINC}
% For pre-training, we sample 2 million unlabeled molecular graphs from the ZINC15 database~\cite{ZINC}. 
For downstream evaluation, we adopt seven widely used molecular property prediction datasets from the chemistry domain~\cite{Hu2020Strategies}. 
These datasets cover a diverse set of biochemical tasks, including drug toxicity prediction, bioactivity classification, and side-effect identification. 
In all datasets, each sample is represented as a molecular graph, where nodes correspond to atoms and edges denote chemical bonds extracted from the molecular structure. 
% The goal is to predict molecular properties such as toxicity (e.g., BACE, BBBP, ClinTox), HIV activity (HIV), side-effect classification (Sider) or biochemical toxicity endpoints (Tox21, ToxCast). %, or multi-assay bioactivity (MUV).
Detailed dataset statistics are summarized in Table~\ref{tab:dataset}.
All datasets use the standard scaffold-based data splits commonly employed in molecular graph learning benchmarks~\cite{graphcl}.
% \begin{table*}[t]
% \centering
% \caption{Statistics of datasets for downstream tasks.}
% \begin{tabular}{lcccccccc} 
% \toprule
% \textbf{Dataset} & BACE & BBBP & ClinTox & HIV & Sider & Tox21 & MUV & ToxCast  \\
% \midrule
% \#Graphs & 1513 & 2039 & 1478 & 41127 & 1427 & 7831 & 93087 & 8575  \\
% \#Avg.Node & 34.09 &24.06 & 26.16& 25.51 & 33.64 & 18.57& 24.23& 18.78 \\
% \#Avg.Edge & 73.72 & 51.91 & 55.77& 54.94 & 70.72&38.59 & 52.56 & 38.52 \\
% \#Tasks  & 1 & 1 & 2 & 1 & 27 & 12 & 17 & 617 \\
% \bottomrule
% \end{tabular}\label{tab:dataset}
% \end{table*}
\begin{table}[!htbp]
\centering
\small
\renewcommand{\arraystretch}{1.1}
\caption{Statistics of datasets for downstream tasks.}
\begin{tabular}{lcccc}
\toprule
\textbf{Dataset} & \#Graphs & \#Avg.Node & \#Avg.Edge & \#Tasks \\
\midrule
BACE    & 1513  & 34.09 & 73.72 & 1   \\
BBBP    & 2039  & 24.06 & 51.91 & 1   \\
ClinTox & 1478  & 26.16 & 55.77 & 2   \\
HIV     & 41127 & 25.51 & 54.94 & 1   \\
Sider   & 1427  & 33.64 & 70.72 & 27  \\
Tox21   & 7831  & 18.57 & 38.59 & 12  \\
ToxCast & 8575  & 18.78 & 38.52 & 617 \\
\bottomrule
\end{tabular}
\label{tab:dataset}
\end{table}
% \begin{table*}[t]
% \centering
% \caption{Statistics of datasets for downstream tasks.}
% \begin{tabular}{lccccccc} 
% \toprule
% \textbf{Dataset} & BACE & BBBP & ClinTox & HIV & Sider & Tox21 &ToxCast  \\
% \midrule
% \#Graphs & 1513 & 2039 & 1478 & 41127 & 1427 & 7831 & 8575  \\
% \#Avg.Node & 34.09 &24.06 & 26.16& 25.51 & 33.64 & 18.57&  18.78 \\
% \#Avg.Edge & 73.72 & 51.91 & 55.77& 54.94 & 70.72&38.59 &  38.52 \\
% \#Tasks  & 1 & 1 & 2 & 1 & 27 & 12  & 617 \\
% \bottomrule
% \end{tabular}\label{tab:dataset}
% \end{table*}

\begin{table*}[htbp]
  \centering
  \caption{Comparison of robustness performance across seven datasets under node feature attack. Bold and underlined indicate the best and second-best methods, respectively. `Avg.' represents the average across all datasets.}
  \label{tab:result_node}
  \small
  \renewcommand{\arraystretch}{1.2}
  \begin{tabular}{l|ccccccc|c}
    \toprule
    Methods & BACE & BBBP & TOX21 & TOXCAST & SIDER & HIV & CLINTOX & \textbf{Avg.} \\
    \midrule
    FT (100\%) & 61.66{\scriptsize$\pm$1.75} & 58.99{\scriptsize$\pm$1.20} & 68.36{\scriptsize$\pm$0.35} & 60.85{\scriptsize$\pm$0.15} & 56.49{\scriptsize$\pm$0.95} & 66.88{\scriptsize$\pm$1.50} & 65.47{\scriptsize$\pm$2.28} & 62.67{\scriptsize$\pm$1.17} \\
    BitFit (0.2\%) & 47.14{\scriptsize$\pm$2.96} & 53.67{\scriptsize$\pm$1.18} & 64.08{\scriptsize$\pm$0.70} & 57.74{\scriptsize$\pm$0.73} & 53.95{\scriptsize$\pm$0.99} & 59.42{\scriptsize$\pm$2.31} & 57.76{\scriptsize$\pm$1.67} & 56.25{\scriptsize$\pm$1.51} \\
    GPF (0.01\%) & 48.53{\scriptsize$\pm$2.11} & 53.29{\scriptsize$\pm$1.40} & 64.24{\scriptsize$\pm$1.14} & 57.80{\scriptsize$\pm$0.36} & 54.41{\scriptsize$\pm$0.42} & 59.64{\scriptsize$\pm$2.02} & 57.55{\scriptsize$\pm$4.85} & 56.49{\scriptsize$\pm$1.76} \\
    LoRA (5.0\%) & 55.85{\scriptsize$\pm$1.92} & 55.82{\scriptsize$\pm$0.60} & 67.04{\scriptsize$\pm$1.04} & 59.43{\scriptsize$\pm$0.50} & 54.79{\scriptsize$\pm$1.25} & 59.81{\scriptsize$\pm$0.99} & 62.69{\scriptsize$\pm$3.68} & 59.35{\scriptsize$\pm$1.43} \\
    AdapterGNN (5.7\%) & 62.40{\scriptsize$\pm$1.38} & \underline{59.58{\scriptsize$\pm$1.06}} & 68.40{\scriptsize$\pm$0.44} & 60.20{\scriptsize$\pm$0.37} & 55.26{\scriptsize$\pm$0.55} & 65.34{\scriptsize$\pm$0.76} & \underline{66.10{\scriptsize$\pm$2.58}} & 62.47{\scriptsize$\pm$1.02} \\
    \midrule
    AGP-S (0.5\%) & \underline{70.75{\scriptsize$\pm$2.42}} & 59.40{\scriptsize$\pm$1.06} & \underline{73.83{\scriptsize$\pm$0.24}} & \underline{62.94{\scriptsize$\pm$0.59}} & \underline{60.17{\scriptsize$\pm$0.49}} & \underline{69.26{\scriptsize$\pm$0.97}} & 65.17{\scriptsize$\pm$1.48} & \underline{65.93{\scriptsize$\pm$1.04}} \\
    AGP (2.6\%) & \textbf{77.41{\scriptsize$\pm$2.24}} & \textbf{65.00{\scriptsize$\pm$0.88}} & \textbf{75.46{\scriptsize$\pm$0.46}} & \textbf{63.92{\scriptsize$\pm$0.44}} & \textbf{60.70{\scriptsize$\pm$0.62}} & \textbf{75.60{\scriptsize$\pm$0.49}} & \textbf{69.12{\scriptsize$\pm$2.56}} & \textbf{69.60{\scriptsize$\pm$1.10}} \\
    \bottomrule
  \end{tabular}
\end{table*}

\begin{table*}[htbp]
  \centering
  \caption{Comparison of robustness performance across seven datasets under topology attacks. Bold indicates the best method. `Avg.' represents the average across all datasets.}
  \label{tab:result_topo}
  \small
  \renewcommand{\arraystretch}{1.2}
  \begin{tabular}{l|ccccccc|c}
    \toprule
    Methods & BACE & BBBP & TOX21 & TOXCAST & SIDER & HIV & CLINTOX & \textbf{Avg.} \\
    \midrule
    FT (100\%) & 74.06{\scriptsize$\pm$1.25} & 60.50{\scriptsize$\pm$2.53} & 70.33{\scriptsize$\pm$0.63} & 60.63{\scriptsize$\pm$0.21} & 59.66{\scriptsize$\pm$1.49} & 67.74{\scriptsize$\pm$0.68} & 57.70{\scriptsize$\pm$1.94} & 64.37{\scriptsize$\pm$1.25} \\
    BitFit (0.2\%) & 73.82{\scriptsize$\pm$1.57} & 59.95{\scriptsize$\pm$1.58} & 71.16{\scriptsize$\pm$0.32} & 60.86{\scriptsize$\pm$0.53} & 62.49{\scriptsize$\pm$0.85} & 65.02{\scriptsize$\pm$1.50} & 56.32{\scriptsize$\pm$1.73} & 64.23{\scriptsize$\pm$1.15} \\
    GPF (0.01\%) & 74.67{\scriptsize$\pm$1.91} & 59.71{\scriptsize$\pm$0.62} & 71.48{\scriptsize$\pm$0.35} & 60.49{\scriptsize$\pm$0.29} & 62.57{\scriptsize$\pm$0.84} & 65.17{\scriptsize$\pm$1.20} & 56.99{\scriptsize$\pm$3.09} & 64.44{\scriptsize$\pm$1.19} \\
    LoRA (5.0\%) & 74.87{\scriptsize$\pm$1.16} & 62.52{\scriptsize$\pm$1.04} & 70.87{\scriptsize$\pm$0.49} & 61.15{\scriptsize$\pm$0.41} & 61.67{\scriptsize$\pm$1.15} & 65.05{\scriptsize$\pm$0.83} & 57.10{\scriptsize$\pm$4.54} & 64.75{\scriptsize$\pm$1.37} \\
    AdapterGNN (5.7\%) & 75.14{\scriptsize$\pm$1.51} & 62.31{\scriptsize$\pm$2.08} & 70.23{\scriptsize$\pm$0.00} & 60.25{\scriptsize$\pm$0.54} & 58.92{\scriptsize$\pm$0.99} & 67.77{\scriptsize$\pm$1.14} & 60.90{\scriptsize$\pm$1.94} & 65.07{\scriptsize$\pm$1.17} \\
    \midrule
    AGP (2.6\%) & \textbf{81.24}{\scriptsize$\pm$1.76} & \textbf{70.47}{\scriptsize$\pm$1.18} & \textbf{75.87}{\scriptsize$\pm$0.35} & \textbf{63.80}{\scriptsize$\pm$0.68} & \textbf{63.77}{\scriptsize$\pm$1.28} & \textbf{79.55}{\scriptsize$\pm$1.52} & \textbf{71.93}{\scriptsize$\pm$3.74} & \textbf{72.38}{\scriptsize$\pm$1.50} \\
    \bottomrule
  \end{tabular}
\end{table*}

\begin{table*}[htbp]
  \centering
  \caption{Comparison of robustness performance across seven datasets under hybrid attacks. Bold indicates the best method. `Avg.' represents the average across all datasets.}
  \label{tab:result_hybrid}
  \small
  \renewcommand{\arraystretch}{1.2}
  \begin{tabular}{l|ccccccc|c}
    \toprule
    Methods & BACE & BBBP & Tox21 & ToxCast & SIDER & HIV & ClinTox & \textbf{Avg.} \\
    \midrule
    FT (100\%) & 59.42{\scriptsize$\pm$1.66} & 54.38{\scriptsize$\pm$1.54} & 64.11{\scriptsize$\pm$0.37} & 57.70{\scriptsize$\pm$0.18} & 55.40{\scriptsize$\pm$1.38} & 60.46{\scriptsize$\pm$1.07} & 50.03{\scriptsize$\pm$2.52} & 57.36{\scriptsize$\pm$1.25} \\
    BitFit (0.2\%) & 49.10{\scriptsize$\pm$2.58} & 50.75{\scriptsize$\pm$1.32} & 60.85{\scriptsize$\pm$0.60} & 54.99{\scriptsize$\pm$0.63} & 56.46{\scriptsize$\pm$0.53} & 54.61{\scriptsize$\pm$1.84} & 49.09{\scriptsize$\pm$2.52} & 53.69{\scriptsize$\pm$1.43} \\
    GPF (0.01\%) & 49.75{\scriptsize$\pm$1.75} & 50.75{\scriptsize$\pm$0.99} & 60.73{\scriptsize$\pm$0.98} & 54.89{\scriptsize$\pm$0.40} & 56.34{\scriptsize$\pm$0.79} & 54.51{\scriptsize$\pm$1.96} & 47.98{\scriptsize$\pm$1.82} & 53.56{\scriptsize$\pm$1.24} \\
    LoRA (5.0\%) & 54.86{\scriptsize$\pm$1.80} & 52.94{\scriptsize$\pm$0.44} & 62.84{\scriptsize$\pm$0.83} & 56.79{\scriptsize$\pm$0.30} & 56.29{\scriptsize$\pm$1.68} & 54.90{\scriptsize$\pm$0.80} & 51.31{\scriptsize$\pm$3.89} & 55.69{\scriptsize$\pm$1.39} \\
    AdapterGNN(5.0\%) & 60.34{\scriptsize$\pm$1.85} & 57.20{\scriptsize$\pm$1.36} & 64.10{\scriptsize$\pm$0.44} & 57.20{\scriptsize$\pm$0.40} & 54.38{\scriptsize$\pm$1.18} & 59.63{\scriptsize$\pm$0.93} & 56.28{\scriptsize$\pm$2.17} & 58.45{\scriptsize$\pm$1.19} \\
    \midrule
    AGP (2.6\%) & \textbf{76.39{\scriptsize$\pm$1.42}} & \textbf{66.18{\scriptsize$\pm$1.12}} & \textbf{73.88{\scriptsize$\pm$0.52}} & \textbf{62.86{\scriptsize$\pm$0.69}} & \textbf{57.07{\scriptsize$\pm$1.10}} & \textbf{75.27{\scriptsize$\pm$0.87}} & \textbf{70.86{\scriptsize$\pm$2.03}} & \textbf{68.93{\scriptsize$\pm$1.11}} \\
    \bottomrule
  \end{tabular}
\end{table*}

\subsubsection{Baseline Methods}
We compare the proposed approach against 
five representative fine-tuning methods:
% several widely-used efficient fine-tuning methods:
\begin{itemize}
    \item Full Fine-Tuning (FT). It is the traditional tuning method that updates all model parameters during fine-tuning phrase.
    
    \item BitFit~\cite{zaken2022bitfit}: It only fine-tunes the bias terms in the model while keeping the pre-trained weights fixed.
    
    \item Graph Prompt Fine-tuning (GPF)~\cite{GPF}. This method introduces a learnable prompt vector, which is added to each node's original features. During fine-tuning, only the parameters of the prompt vector are updated.

    \item LoRA~\cite{hulora}.
   This method employs a parallel LoRA module with a bottleneck architecture. Consistent with the setup used in prior work~\cite{AdapterGNN}, the LoRA module is positioned both before and after the message passing operation in each GNN layer.
    
    \item AdapterGNN~\cite{AdapterGNN}:
  This method introduces dual adapter modules in parallel with the pre-trained GNN. It fine-tunes the parameters of the bias terms in the MLP, BatchNorm layers, and the adapter modules.
    
\end{itemize}
For all baseline methods, we adopt the optimal hyperparameter settings reported in their respective original publications. To ensure a fair comparison, all baselines are evaluated after fine-tuning for 100 epochs.
% To further evaluate robustness, we incorporate four data preprocessing-based defense strategies: two methods (A and B) designed for feature denoising, and two methods (C and D) for graph structure denoising. All FT and PEFT methods are evaluated under these four defense settings. In the results, “GPF (PCA)” denotes applying PCA-based feature denoising before executing GPF.

\subsubsection{Attack Settings}
To comprehensively evaluate the robustness of AGP,
we test the ROC-AUC of model under three different adversarial noise scenarios:  
\begin{itemize}
    \item Node Attack: Implemented using the standard PGD method~\cite{PGD} with perturb step $K =10$ and perturb radius $\epsilon= 0.8$. For dataset BBBP and CLINTOX, we set the learn rate $\alpha= 0.005$. 
    For other datasets, we set the learn rate $\alpha= 0.01$. 

    \item Topology Attack: Conducted using a discrete PGD algorithm~\cite{GraphPGD} with perturb step $K =10$ and perturbation ratio $r = 40\%$. For dataset BBBP and CLINTOX, we set the learn rate $\beta= 30$. For other datasets, we set the learn rate $\beta =100$. 
    
    \item Hybrid Attacks: Performed using the proposed JointPGD method, with parameters consistent with the individual node and topology attacks above.
\end{itemize}
To assess the model's effectiveness, we adopt `ROC-AUC' as the primary evaluation metric.
% The model's performance is evaluated using the ROC-AUC metric. 
We report results both on clean data and adversarial data to measure standard performance and robustness, respectively.
Due to the randomness in the initial noise matrix for node attacks and the probabilistic sampling step in topology attacks, we execute each attack five times under every network initialization and report the average results.

\subsubsection{Model Settings}

For graph classification on chemical datasets, we adopt a 5-layer GIN~\cite{GIN} backbone with `mean' pooling strategy, following common practice in prior works~\cite{GPF,AdapterGNN}. 
Independent prompt modules are inserted at the input of each GIN layer to help defend against adversarial noise.
A three-layer MLP is added at the end of the model for downstream classification. 
The hidden dimension is set to 300 and the bottleneck dimension in each prompt module is set to 64. 
During fine-tuning, the parameters of the pretrained GIN encoder are frozen, while only the graph prompt modules and the classifier are updated, as shown in Fig.~\ref{fig:architecture}. 
% We employ the Adam optimizer with a learning rate of 0.001 and weight decay of 0.
Model is trained by the Adam optimizer~\cite{Adam}, configured with a learning rate of $0.001$ and a weight decay coefficient of $0$.
To ensure stable adversarial optimization, AGP is trained with 30 warm-up epochs on clean data to initialize the prompt parameters before adversarial fine-tuning. 
The hyperparameters $\gamma$ and $\eta$ are set to 0.3 and 0.6, respectively. 
For fair comparison, the pre-trained model are initialized through Infomax strategy~\cite{infograph} from the same publicly available pretrained GIN model~\cite{Hu2020Strategies}, which can be downloaded from the website \footnote{https://github.com/snap-stanford/pretrain-gnns\label{foot1}}.

\subsection{Comparison results}

Table \ref{tab:result_node}-\ref{tab:result_acc} report the performance of full-parameter fine-tuning and four PEFT strategies across seven molecular property prediction datasets under node, topology and hybrid attacks, respectively. 
% Remarkably, the proposed AGP not only establishes superior average performance against attacks but also gain improvements in average performance on the clean datasets.
Notably, the proposed AGP not only achieves superior average robustness against various attacks but also gains improvements in generalization performance on clean datasets.
% Overall, we observe that the proposed AGP and AGP-S outperform all baselines on the seven datasets on average by a clear margin in Rob., while simultaneously maintaining strong Acc. performance.
To be specific,
(i) across all three types of adversarial attacks, the proposed AGP consistently achieves the highest robustness on most datasets, clearly outperforming existing fine-tuning and parameter-efficient tuning approaches. 
This trend demonstrates that incorporating adversarial training into the PEFT framework can substantially enhance model robustness, enabling the model to better resist both node and topology attacks.
(ii) AGP also provides competitive or improved accuracy on original clean datasets.
We attribute these marginal improvements to two primary factors. First, the inherent noise within the original datasets may be effectively mitigated by our robust prompting. Second, for downstream tasks with limited sample sizes, the generated adversarial examples serve as an effective form of data augmentation, thereby enhancing learning performance. 
% This slight increase in clean accuracy suggests that adversarial training may help the model alleviate the impact of inherent noise present in molecular property datasets. As a result, AGP achieves better generalization without sacrificing performance.
(iii) 
For AGP-S (the single-prompt-layer variant of AGP), the results show that it also enhances robustness under node attack (Table~\ref{tab:result_node}), which is consistent with the analysis in $\S V$. 
However, due to the limited learnable parameters, its average performance gains in terms of robustness and accuracy are slightly lower than the proposed AGP. 
% This performance gap stems primarily from the limited representational capacity inherent in the single-layer design.

\begin{table*}[!htbp]
  \centering
  \caption{Comparison of ROC-AUC results across seven datasets. `AGP-Node', `AGP-Topology', and `AGP-Hybrid' denote the results obtained by AGP through adversarial training under node feature attack, graph topology attack, and hybrid attack scenarios, respectively. Bold indicates the best method. `Avg.' represents the average across all datasets.
  }
  \label{tab:result_acc}
  \small
  \renewcommand{\arraystretch}{1.2}
  \begin{tabular}{l|ccccccc|c}
    \toprule
    Methods & BACE & BBBP & TOX21 & TOXCAST & SIDER & HIV & CLINTOX & \textbf{Avg.} \\
    \midrule
    FT (100\%)  & 80.68{\scriptsize$\pm$1.10} & 67.59{\scriptsize$\pm$1.13} & 76.24{\scriptsize$\pm$0.49} & 64.49{\scriptsize$\pm$0.24} & 63.96{\scriptsize$\pm$1.09} & 75.66{\scriptsize$\pm$1.08} & 74.77{\scriptsize$\pm$1.76} & 71.91{\scriptsize$\pm$0.98} \\
    BitFit (0.2\%) & 79.34{\scriptsize$\pm$1.26} & 65.07{\scriptsize$\pm$0.63} & 77.46{\scriptsize$\pm$0.44} & 65.32{\scriptsize$\pm$0.57} & 65.20{\scriptsize$\pm$0.98} & 72.40{\scriptsize$\pm$1.98} & 72.77{\scriptsize$\pm$1.93} & 71.08{\scriptsize$\pm$1.11} \\
    GPF (0.01\%) & 80.80{\scriptsize$\pm$1.89} & 64.69{\scriptsize$\pm$0.88} & 77.89{\scriptsize$\pm$0.62} & 65.23{\scriptsize$\pm$0.35} & 65.93{\scriptsize$\pm$0.56} & 72.96{\scriptsize$\pm$1.25} & 72.50{\scriptsize$\pm$3.93} & 71.43{\scriptsize$\pm$1.35} \\
    LoRA (5.0\%) & 81.17{\scriptsize$\pm$0.94} & 66.26{\scriptsize$\pm$0.69} & 77.57{\scriptsize$\pm$0.70} & 65.08{\scriptsize$\pm$0.48} & 65.42{\scriptsize$\pm$1.27} & 72.20{\scriptsize$\pm$0.95} & 74.61{\scriptsize$\pm$3.67} & 71.76{\scriptsize$\pm$1.24} \\
    AdapterGNN (5.0\%) & 81.58{\scriptsize$\pm$1.16} & 67.90{\scriptsize$\pm$1.14} & 76.32{\scriptsize$\pm$0.56} & 64.28{\scriptsize$\pm$0.48} & 63.51{\scriptsize$\pm$0.60} & 75.32{\scriptsize$\pm$0.93} & 76.55{\scriptsize$\pm$3.29} & 72.21{\scriptsize$\pm$1.17} \\
    \midrule
    AGP-S (0.5\%)  & 82.03{\scriptsize$\pm$1.30} & 67.27{\scriptsize$\pm$0.69} & \textbf{78.15{\scriptsize$\pm$0.32}} & \textbf{66.16{\scriptsize$\pm$0.54}} & \textbf{67.11{\scriptsize$\pm$0.24}} & 75.92{\scriptsize$\pm$0.64} & 73.08{\scriptsize$\pm$1.58} & 72.82{\scriptsize$\pm$0.76} \\
    AGP-Node (2.6\%) & \textbf{83.00{\scriptsize$\pm$1.73}} & \textbf{70.40{\scriptsize$\pm$0.59}} & 77.82{\scriptsize$\pm$0.31} & 65.53{\scriptsize$\pm$0.42} & 64.17{\scriptsize$\pm$0.58} & \textbf{78.20{\scriptsize$\pm$0.42}} & 76.45{\scriptsize$\pm$2.34} & \textbf{73.65{\scriptsize$\pm$0.91}} \\
    AGP-Topology (2.6\%)& 82.12{\scriptsize$\pm$0.96} & 70.24{\scriptsize$\pm$1.20} & 77.19{\scriptsize$\pm$0.19} & 65.34{\scriptsize$\pm$0.32} & 65.47{\scriptsize$\pm$0.89} & 77.47{\scriptsize$\pm$1.24} & 75.47{\scriptsize$\pm$3.54} & 73.33{\scriptsize$\pm$1.19} \\
    AGP-Hybrid (2.6\%)& 82.00{\scriptsize$\pm$0.77} & 70.38{\scriptsize$\pm$0.39} & 77.64{\scriptsize$\pm$0.39} & 65.14{\scriptsize$\pm$0.60} & 64.58{\scriptsize$\pm$1.74} & 76.21{\scriptsize$\pm$1.19} & \textbf{77.05{\scriptsize$\pm$2.18}} & 73.29{\scriptsize$\pm$1.04} \\
    \bottomrule
  \end{tabular}
\end{table*}

\subsection{Model analysis}
\subsubsection{Pre-train model}
\begin{figure*}[!htpb]
\centering
\includegraphics[width=1.0\textwidth]{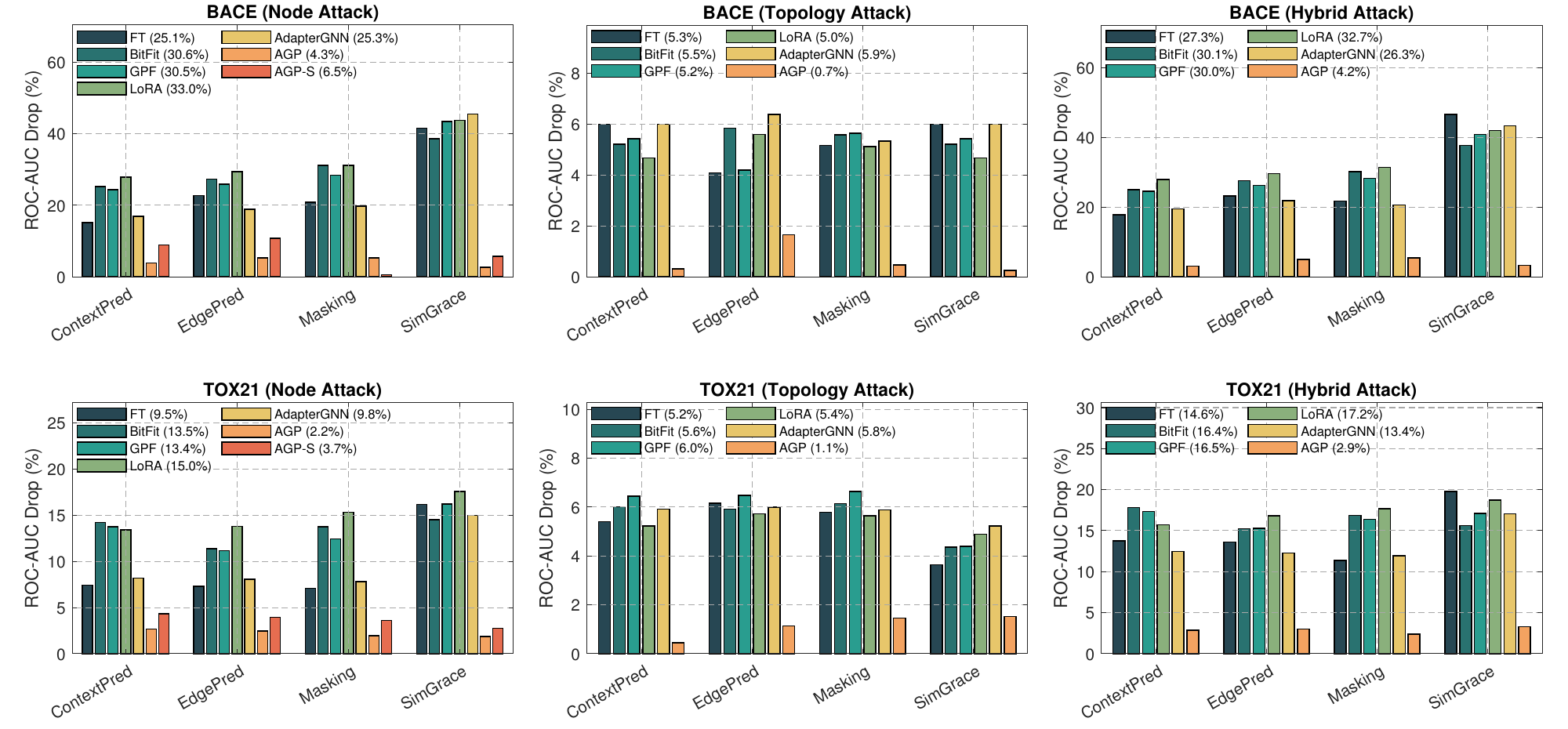}
 \caption{Robustness evaluation of fine-tuning strategies across varying pre-training models on BACE and TOX21 datasets. The legend displays the average results under various pre-training strategies.
 }
\label{fig:pretrain}
\end{figure*}
To verify the universality of our approach with respect to various pre-trained backbones, we evaluate the robustness against node, topology, and hybrid attacks on the BACE and TOX21 datasets with four different pre-training methods, including ContextPred~\cite{Hu2020Strategies}, EdgePred~\cite{graphsage}, Masking~\cite{Hu2020Strategies} and SimGrace~\cite{simgrace}.
To ensure a fair comparison, all pre-trained model parameters are sourced from publicly available websites\textsuperscript{\ref{foot1}}\footnote{https://github.com/junxia97/SimGRACE}.
The robustness is quantified by the `ROC-AUC Drop', defined as the performance degradation from the clean data to the attacked data. The lower value indicates higher robustness.
As visualized in Fig.~\ref{fig:pretrain}, a key observation is that standard fine-tuning (FT) and existing PEFT methods (e.g., GPF~\cite{GPF}, LoRA~\cite{hulora}, and AdapterGNN~\cite{AdapterGNN}) are highly susceptible to three kinds of adversarial attacks across all pre-training models.
% Regardless of whether the backbone is initialized via ContextPred~\cite{contextpred}, EdgePred~\cite{edgepred}, Masking~\cite{masking}, or SimGrace~\cite{simgrace}, these baseline methods consistently suffer from severe performance degradation under adversarial attacks. 
% For instance, under the hybrid attack, all baseline model consistently suffer a performance loss of over 20\% in ROC-AUC on the BACE dataset.
For instance, on the BACE dataset, a hybrid attack typically causes the performance of existing models to drop by more than 20\%.
In contrast, the proposed AGP and AGP-S methods demonstrate universally superior robustness across all pre-training models. 
Even under the same attack conditions, they keep the drop in ROC-AUC within a 5\% range on most datasets.
This result indicates that AGP provides a generalized robustness enhancement across different pre-training scenarios.

\subsubsection{Efficiency analysis}
To analyze the efficiency of our proposed AGP, we compare it with baseline methods on the number of tunable parameters, GPU memory usage and inference latency. We provide the ratio of tunable parameters to fully fine-tune parameters for each method in Table~\ref{tab:result_node}-\ref{tab:result_acc}. Fig.~\ref{fig:effecy} shows the GPU memory usage and inference latency of each method on TOX21 and HIV datasets. 
In terms of tunable parameters, AGP requires only 2.6\% of the full-tuning (FT) parameters, while the simplified single-layer variant AGP-S further reduces this ratio to 0.5\%, which is lower than common PEFT baselines such as LoRA (5.0\%) and AdapterGNN (5.7\%). 
In terms of computational efficiency, both AGP and AGP-S demonstrate competitive performance regarding GPU memory usage and inference latency. Specifically, AGP-S achieves inference times of 0.67$\mu$s and 0.65$\mu$s, which are comparable to those of the lightweight baseline GPF (0.67$\mu$s and 0.64$\mu$s) on TOX21 and HIV datasets. While AGP incurs a marginal increase in latency due to its multi-layer architecture, it remains highly efficient. Furthermore, AGP-S maintains a low memory usage, aligning closely with three lightweight methods (FT, GPF and BitFit). 
Moreover, when compared against other multi-layer architectures like LoRA and AdapterGNN, AGP exhibits competitive GPU memory efficiency.
Overall, 
% these results indicate that the proposed approaches remain highly competitive among existing PEFT methods while offering little parameter overhead.
these results show that the proposed methods can maintain high competitiveness in terms of efficiency among existing PEFT methods.
\begin{figure*}[!htpb]
\centering
\includegraphics[width=1.0\textwidth]{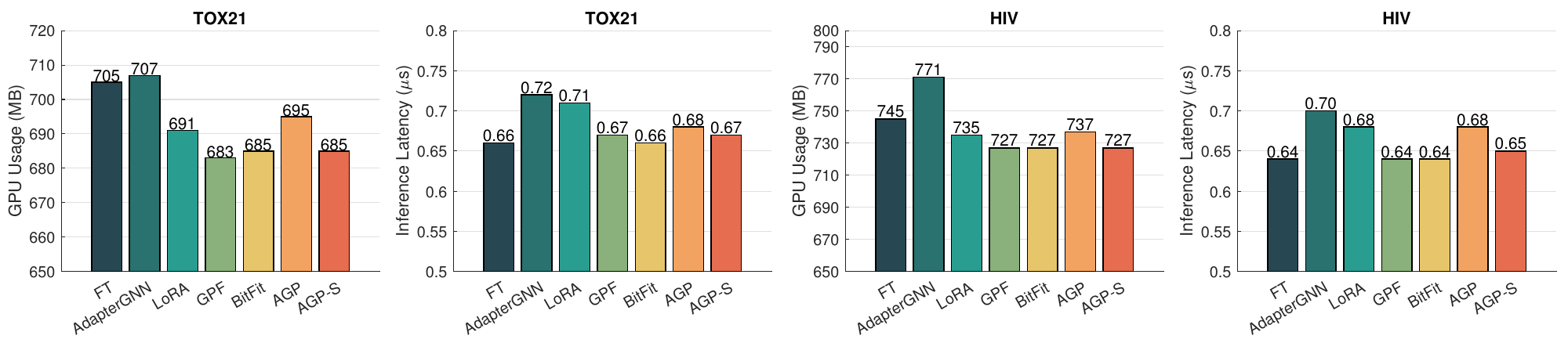}
 \caption{Comparison of GPU memory usages and inference latency across different fine-tuning methods.
 }
\label{fig:effecy}
\end{figure*}

\subsubsection{Bottleneck dimension}

In this section, we investigate the impact of the bottleneck dimension on the robustness and accuracy of AGP. Fig.~\ref{fig:bottleneck} illustrates the performance of AGP across varying bottleneck dimensions on the BACE and TOX21 datasets under hybrid attack. Specifically, an excessively small dimension (e.g., $d=1$) restricts the model's representational capacity, leading to underfitting. Conversely, an overly large dimension (e.g., $d=150$) induces overfitting, resulting in performance degradation. Consequently, a bottleneck dimension of approximately 60 achieves an optimal balance between parameter efficiency and model performance. Notably, even with a highly constrained bottleneck (e.g., $d=1$), AGP surpasses the full fine-tuning baseline on both datasets.
\begin{figure}[!htpb]
\centering
\includegraphics[width=0.5\textwidth]{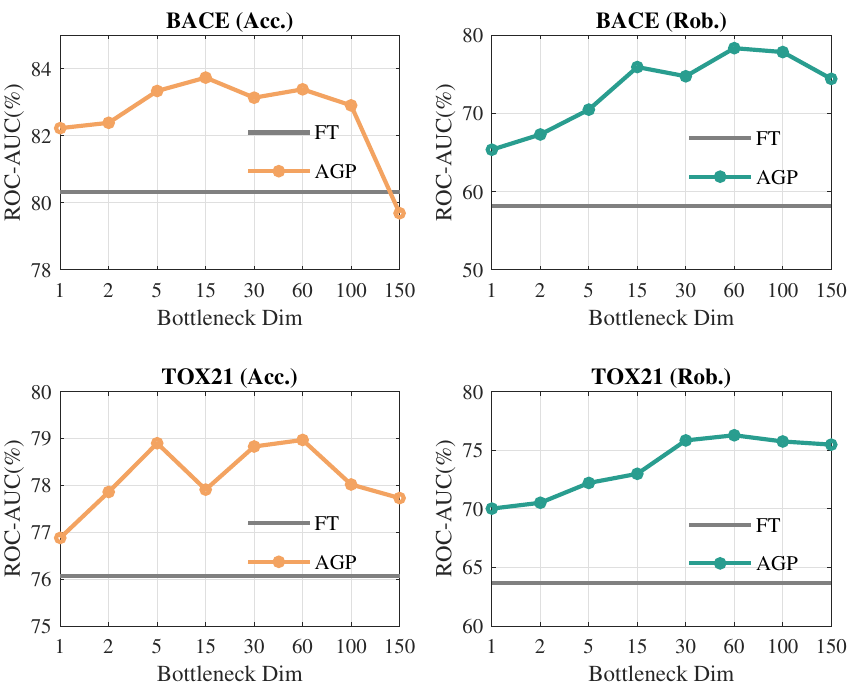}
 \caption{Convergence behavior of the loss function during model training. `Acc.' and `Rob.' denote the model's performance on clean and adversarial graph data, respectively.
 }
\label{fig:bottleneck}
\end{figure}

\subsubsection{Loss analysis}
\begin{figure}[!htpb]
\centering
\includegraphics[width=0.5\textwidth]{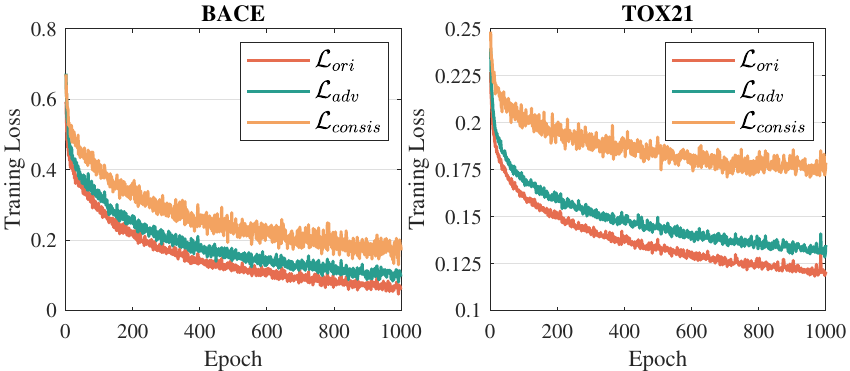}
 \caption{Convergence curves of three different loss functions on the BACE and TOX21 datasets.
 }
\label{fig:loss}
\end{figure}
\begin{table}[h!]
\caption{Ablation study of the three loss components on the BACE and Tox21 datasets. `Acc.' and `Rob.' denote the model's performance on clean and adversarial graph data, respectively.}
\centering
\small
\renewcommand{\arraystretch}{1.3}
\label{tab:loss}
\begin{tabular}{ccc|cc|cc}
\hline  
\hline
\multirow{2}{*}{$\mathcal{L}_{ori}$}&
\multirow{2}{*}{$\mathcal{L}_{adv}$}&
\multirow{2}{*}{$\mathcal{L}_{consis}$}
& \multicolumn{2}{c|}{BACE} 
& \multicolumn{2}{c}{TOX21} \\
\cline{4-7}
 &  &  &Rob. & Acc. & Rob. & Acc. \\
\hline
$\surd$   &    &     & 54.52 & 81.46 & 62.5  & 75.8 \\
    &$\surd$   &     & 77.42 & 78.67 & 72.55 & 75.16 \\
    &     & $\surd$   & 68.97 & 72.47 & 65.48 & 76.17 \\
$\surd$ & $\surd$ &     & 72.71 & 80.32 & 70.97 & 76.66 \\
$\surd$ &     & $\surd$ & 74.32 & 82.51 & 69.06 & 76.99 \\
    & $\surd$ & $\surd$ & \textbf{78.27} & 81.26 & \textbf{74.51} & 76.84 \\
$\surd$ & $\surd$ & $\surd$ & 78.12 & \textbf{82.86} & 74.03 & \textbf{77.99} \\
\hline\hline
\end{tabular}
\end{table}
Fig.~\ref{fig:loss} illustrates the convergence curves of the three distinct loss components on the BACE and TOX21 datasets.
The consistent decrease and eventual stabilization of all curves verify the training stability of the proposed AGP framework.
To further dissect the contribution of each component, Table~\ref{tab:loss} presents the results regarding robustness (`Rob.') and clean accuracy (`Acc.') on BACE and TOX21. 
Several consistent trends can be observed. 
First, training solely with $\mathcal{L}_{ori}$ yields the lowest robustness on both datasets, confirming that clean-only supervision is insufficient to defend against adversarial attacks. 
Second, incorporating $\mathcal{L}_{adv}$ leads to a substantial boost in `Rob.' (e.g., from 54.52\% to 72.71\% on BACE), validating the effectiveness of adversarial supervision.
However, this gain often comes at the cost of reduced `Acc'. 
Crucially, introducing the consistency regularization ($\mathcal{L}_{consis}$) mitigates this trade-off: when combined with either $\mathcal{L}_{ori}$ or $\mathcal{L}_{adv}$, it helps stabilize performance and yields more balanced improvements.
Finally, the full objective, which jointly optimizes all three losses, achieves the best overall performance, delivering competitive `Rob.' and `Acc.' across both benchmarks. These results suggest that the consistency constraint effectively aligns the clean and adversarial domains, enabling the model to retain clean accuracy while enhancing robustness.

% \subsubsection{Adversarial Framework Comparison}
% adv+gpf, adv+adapterGNN, our
% y:training time, x:acc

\section{Conclusion}
In this paper, we propose a novel framework 
Adversarial Graph Prompting (AGP) to improve the robustness of parameter-efficient GNN fine-tuning against hybrid adversarial noise. 
The proposed AGP is formulated as a min-max optimization problem and addressed by iteratively optimizing between inner maximization and outer minimization.   
We have provided a theoretical analysis to illustrate how AGP mitigates noise in both graph topology and node features.
% we integrate the Joint Projected Gradient Descent (JointPGD) strategy within a min-max optimization formulation. This mechanism allows our approach to effectively generate robust prompts that counteract both topology and node feature noise. 
Extensive experiments demonstrate that AGP significantly outperforms five representative fine-tuning methods in terms of both robustness and accuracy.

% on robustness and accuracy simultaneously. 
% Crucially, our method successfully improves adversarial robustness while simultaneously maintaining high prediction accuracy.

% \section*{Acknowledgments}
% xxxxxx


\bibliography{Reference}
\bibliographystyle{ieeetr}

\end{document}